\documentclass[11pt]{article}

\usepackage[final]{acl}

\usepackage{times}
\usepackage{latexsym}

\usepackage[T1]{fontenc}

\usepackage[utf8]{inputenc}

\usepackage{microtype}

\usepackage{inconsolata}

\usepackage{graphicx}
\usepackage{wrapfig}
\usepackage{hyperref}
\usepackage{enumitem}
\usepackage{booktabs}
\usepackage{longtable}
\usepackage{lineno}
\usepackage{caption}
\usepackage{listings}
\usepackage{xcolor}

\usepackage{tabularx}
\usepackage{amsmath,amssymb,amsfonts}%
\usepackage{amsthm}%

\usepackage{natbib}
\usepackage{doi}
\usepackage{makecell}
\usepackage{mathbbol}
\usepackage{graphicx}
\usepackage{tabularx}
\usepackage{booktabs}
\usepackage{array}
\usepackage{colortbl}
\usepackage{pdflscape}
\usepackage{tabu}
\usepackage{threeparttable}
\usepackage{threeparttablex}
\usepackage[normalem]{ulem}
\usepackage{makecell}
\usepackage{multirow}%
\usepackage{multicol}%

\theoremstyle{plain}
\newtheorem{theorem}{Theorem}[section]
\newtheorem{proposition}[theorem]{Proposition}

\newtheorem{corollary}[theorem]{Corollary}
\theoremstyle{definition}
\newtheorem{definition}[theorem]{Definition}

\theoremstyle{remark}

\title{Knowledge without Wisdom: Measuring Misalignment between LLMs and Intended Impact}

\author{Michael Hardy \\
  Stanford University \\
  \texttt{hardym[}$\alpha\tau$\texttt{]stanford[}$\circ$\texttt{]edu} \\\And
  Yunsung Kim \\
  Stanford University \\
  \texttt{yunsung[}$\alpha\tau$\texttt{]stanford[}$\circ$\texttt{]edu} \\}

\begin{document}
\maketitle
\begin{abstract}

LLMs increasingly excel on AI benchmarks, but doing so does not guarantee validity for downstream tasks. This study contrasts LLM alignment on benchmarks, downstream tasks, and, importantly the intended impact of those tasks. We evaluate the performance of leading LLMs (i.e., generative pre-trained base models) on difficult-to-verify tasks of the teaching and learning of schoolchildren. Across all LLMs, inter-model behaviors on disparate tasks correlate higher than they do with expert human behaviors on target tasks. These biases shared across LLMs are poorly aligned with downstream measures of teaching quality and often \textit{negatively aligned with the intended impact} of student learning outcomes. Further, we find multi-model ensembles, both unanimous model voting and expert-weighting by benchmark performance, further exacerbate misalignment with learning. We measure that selection of LLM and/or prompting strategy only reliably accounts for 15\% of all measured misalignment error and that variation in misalignment error is shared across LLMs, suggesting that common pretraining accounts for much of the misalignment in these tasks. We demonstrate methods for robustly measuring alignment of complex tasks and provide unique insights into practical applications of LLMs in high-noise contexts.

\end{abstract}

\section{Introduction}
\begin{quote}
\textit{Where is the wisdom we have lost in knowledge? Where is the knowledge we have lost in information?} \citep{eliot_rock_1934}
\end{quote}

Large language models (LLMs) now exhibit striking competence on benchmarks that operationalize \emph{knowledge}: answering questions, reproducing domain vocabulary, and generating fluent explanations. We have also seen rapidly growing optimism about using LLMs for tasks that require more than static Q\&A, such as scientific discovery and hypothesis generation \citep{xin_towards_2025,yamada_ai_2025,gottweis_towards_2025,swanson_virtual_2025,cong_labos_2025}. Yet recent work in this area highlights failure modes that are not well-captured by standard evaluation: overconfidence in interpreting evidence, brittle multi-step inference, and performance plateaus that appear tied to shared pretraining distributions rather than to idiosyncratic architectures \citep{song_evaluating_2025,alampara_probing_2025,mirza_framework_2025,kim_correlated_2025,kalai_why_2025}. Such limitations increase when tasks move downstream, away from ``correct answers'' and toward some \emph{intended impact} in the world.
\begin{figure}
    \centering
    \includegraphics[width=\linewidth]{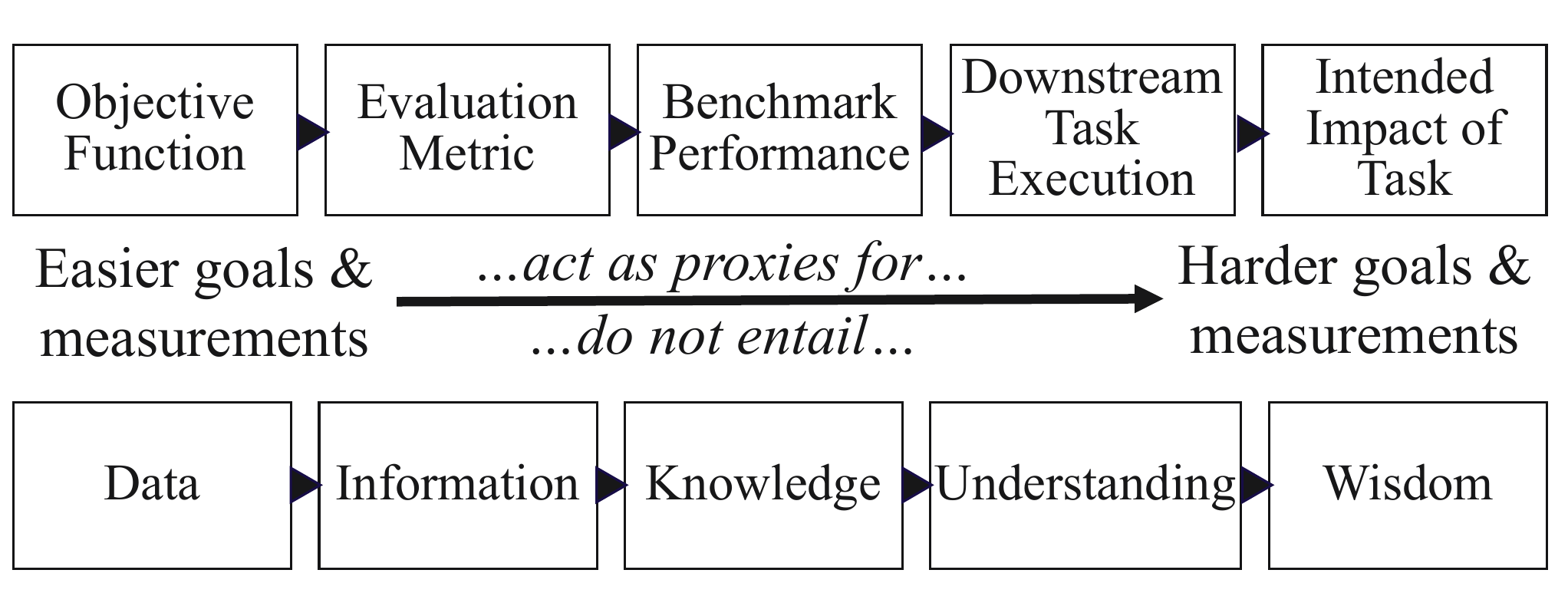}
    \caption{Cascading levels of inference found in LLM development and evaluation and a parallel adaption of Ackoff's progression \citep{ackoff_data_1989,rowley_wisdom_2007}}
    \label{fig:cascading}
\end{figure}

The research of LLMs-in-scientific-discovery studies illustrates a gap between benchmarks and downstream tasks that is not reserved for prestigious tasks that are difficult for scientific experts. This paper uses similarly rigorous methods to study what might be considered a simpler domain than frontier scientific discovery—elementary classrooms—but one that makes the same scientific point sharply: impressive language competence does not guarantee that a model's judgments align with the implied construct of interest. Classroom instruction is an archetypal high-stakes, high-noise setting where quality may be inferred from unstructured discourse text 
and where the ultimate objective is delayed: student learning. Yet, in ways similar to many application areas, annotated classroom discourse is effectively absent from
the Internet text\footnote{see Limitations in \ref{limit}.
} that dominates LLM pretraining, raising generalization questions about what is attributable to pretraining or LLM idiosyncrasies: \emph{what shared behaviors do LLMs exhibit when asked to perform an exceptionally uncommon task with respect to their training data?}


\subsection{From proxy evaluation to intended impact}
The hierarchies of inference in generative AI evaluations can be represented as a series of cascading \textit{proxies}: easier measures acting in place of the true desired outcome (Fig. \ref{fig:cascading}). In children's education settings, common proxy criteria, such as (non-student) user preference \citep{jurenka_towards_2024,kornell_edtech_2024}, are potentially informative but incomplete, because proxies can be optimized without improving schooling's intended impact: student growth. 
This concern mirrors the literature in other disciplines such as coding and scientific-discovery, where strong performance on question-answering benchmarks (e.g., GPQA, \citealt{rein_gpqa_2023}; MMLU-Pro, \citealt{wang_mmlu-pro_2024}) can coexist with weaknesses on tasks that are required to deploy real code or conduct real science \citep{becker_measuring_2025,song_evaluating_2025,alampara_probing_2025,mirza_framework_2025,kim_correlated_2025,kalai_why_2025}. We posit that there are many other domains with congruent gaps. In education, we observe that models may reproduce the \emph{language} of effective pedagogy without tracking the features of instruction that causally support learning.

As it is for such application areas, standard AI benchmarks, which are typically focused on question-answering or tasks with discrete solutions, are ill-suited for the nuanced, generative, and high-stakes nature of educational applications \citep{gehrmann_repairing_2022, hu_prompt-based_2023,zhou_navigating_2023, zhou_relying_2024,kim_im_2024,wu_style_2023,reuel_betterbench_2024, hardy_autoscoring_2026}. In this setting, one could build AI systems that \emph{act} pedagogically sound while failing to identify teaching practices that actually improve achievement, risking deploying technologies that are not only ineffective but potentially harmful to student learning \citep{bastani_generative_2024,shein_impact_2024,rismanchian_artificial_2026, lee_impact_2024}.

To avoid these pitfalls, we use two external criteria that, to our knowledge, have only been linked by one other LLM study \citep{hardy_measuring_2025}:
(i) \textbf{\textit{Downstream Task}, which is expert human observation annotations} using real-world instructional instruments, and
(ii) \textbf{\textit{Intended Impact} through value-added measures (VAMs) of long-term student achievement gains} for those same classrooms. The latter is considered the ``gold standard'' for measuring impact on student learning. Methodologically, we treat both as alignment targets: alignment with the \emph{downstream task} (expert ratings of teaching practice) and alignment to the \emph{intended impact} (predicting which classrooms produce greater learning gains). A primary contribution at the intersection of LLMs and classrooms, it evaluates LLMs using outcome-based criteria rather than human preference alone (see sections \ref{sec:vams} and \ref{sec:rubrics}).
\paragraph{High-level Trends} Across analyses in the present study, a consistent picture emerges that parallels recent results in other disciplines: models can converge on confident, mutually reinforcing judgments that are poorly tethered to the underlying target. In classrooms, ``knowledge'' of pedagogical concepts does not reliably translate into the ``wisdom'' to discern what is relevant to human student learning. 

\subsection{Contributions to High-noise Contexts}
The study provides several important contributions to the current study of applied LLMs. First, it directly quantifies a novel gap between LLM execution of downstream tasks and the \emph{intended impact} (see section \ref{sec:discalign}).
Then, using ensembling, we test whether LLM idiosyncratic competence (by weighting models by ``pedagogy expertise'' on benchmarks) or shared pretraining (by using consensus/unanimity) mitigates misalignment.
Finally, we decompose the variance in misalignment attributable to the two levers practitioners most often control, model choice and prompt choice, with a  method that can be generalized to other alignment studies.

A key contribution is the general methodological framework we use for measuring LLM alignment in high-noise contexts where strong experimental controls may not be feasible:

\begin{enumerate}[nosep]
    \item \textbf{Correlate behaviors across space of generalization}: Demonstrate correlation of behaviors across the types of models, prompting and post-training techniques, environments, and/or harnesses across which we hope to generalize  desired downstream tasks. \S~\ref{sec:dcor}
    \item \textbf{Measure best-proxy alignment}: Measure (mis)alignment on downstream tasks. Kendall's $\tau$ is a simple, strong option that can be estimated with any datatype. \S~\ref{sec:alignment}
    \item \textbf{Determine impact alignment}: Data representing intended impact may have time-delays after the original downstream tasks or may have grain-size mismatches. For high-noise contexts, establish real-world baselines using best available predictors to be able to discern meaningful and practical reference points for comparison. \S~\ref{sec:alignment} 
    \item \textbf{Decompose the error:} decompose misalignment error variance across the space of generalization \citep{meehl_appraising_1990,brennan_generalizability_2001} to quantify contributions to the error. \S~\ref{sec:vardecomp}
\end{enumerate}

\section{Experimental Design}
Using transcripts from primary school mathematics classrooms, we prompt a suite of 16 leading LLMs to assign ordinal ratings based on a rubric (Fig. \ref{fig:expdes}) where each transcript is evaluated across multiple observation dimensions.  We then measure the directional alignment between (a) LLM scores with expert human ratings on the same dimensions and (b) LLM scores with VAMs. Because primary school discourse is effectively absent from pretraining, we can measure how well models generalize to a mismatched distribution.

\begin{figure}
    \centering
    \includegraphics[width=1\linewidth]{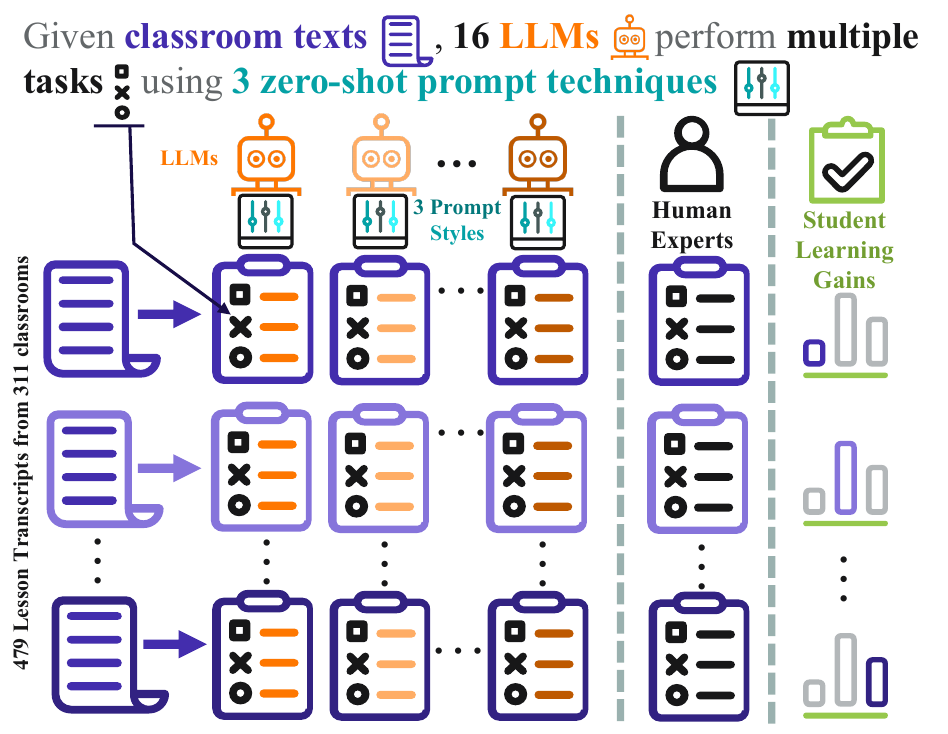}
    \caption{\textbf{Task Data and Experimental Design}. Each LLM is provided classroom transcripts. Using several prompting techniques for each model, LLMs place an ordinal rating on the quality on an aspect of teaching and learning. This is done across seven distinct tasks. We evaluate for alignment of LLM values, and not for accuracy, when comparing the relative ranking provided by each LLM on each task with human experts and with student learning gains for the class.}
    \label{fig:expdes}
\end{figure}

For each classroom lesson in our test set, we provide multiple LLMs with a transcript segment. The models are prompted to perform seven distinct tasks, each focused on a different dimension of teaching and learning (details in Section \ref{sec:data}). We first use the bias-corrected squared distance correlation $\operatorname{dCor}^2_n$ to measure any deviation from independence between the ratings. Our core alignment analysis compares pairwise directionality: 
we assess whether expert human ratings or student learning data also lesson pairs in the same order as LLMs. This approach, which evaluates the alignment of second-moment rank-based information, provides a direct and robust measure of alignment that is insensitive to variations in rating distributions and allows for authentic baseline comparisons. We then use multiple ensemble methods to amplify/attenuate shared misalignment signals. To understand the sources of misalignment, we structurally decompose observed misalignment errors by the variables in our space of generalization: LLM, prompting strategy, item/task, and class transcript.


\section{Methods}
Even among other high-noise contexts, evaluation of educational applications is particularly challenging \cite{kraft_interpreting_2020, jurenka_towards_2024} and, at best, is based on noisy instruments \cite{mccaffrey_intertemporal_2009, kane_gathering_2012, kane_have_2013,hardy_all_2024} measuring latent constructs \cite{messick_validity_1995, hill_validating_2012} and questionable data quality \cite{ho_reliability_2013, xu_promises_2024}. Our methodology is designed to measure the alignment between the relative ordering of teaching quality as judged by LLMs, expert humans, and student learning outcomes. We eschew direct comparisons of absolute rating scores, which are susceptible to noise and idiosyncratic scale use by both humans and models. Instead, we treat LLM ordinal outputs as Thurstonian indicators of their pedagogical values by focusing on pairwise concordance, a robust measure of directional agreement, inspired by research on AI safety and utility \cite{mazeika_utility_2025, huang_values_2025}.

\subsection{Measuring dependence with $\text{dCor}^2_n$}\label{sec:dcor}
To measure any amount of dependence between observations, including nonlinear and nonmonotonic, we use the Bias Corrected Squared Distance Correlation \citep{szekely_measuring_2007,szekely_partial_2014} to determine the strength of the relationships between tasks and raters. We report the average squared correlations disaggregated by relationship type in Figure \ref{fig:dCor_summary}. Inter-LLM specific patterns of shared behavior are more visually striking in Figs. \ref{fig:distance} and \ref{fig:distance_class}. Additional details are in Appendix \ref{app:dcor}.

\subsection{Measuring alignment with Kendall's \texorpdfstring{$\tau$}{Tau}}\label{sec:alignment}
To formalize this concept of alignment, we employ Kendall's $\tau$, a nonparametric and robust  measure of concordance \cite{kendall_treatment_1945,bishara_confidence_2017}. We reconstruct its formulation here to motivate it as an intuitive and precise measure of alignment between two sets of scores, such as those from an LLM, $x$, and an outcome metric, $y$ (e.g., expert scores or student learning gains). Consider a set of $n$ lessons. For any pair of distinct lessons, indexed by $i$ and $j$, we can evaluate whether the ratings from source $x$ and source $y$ agree on their relative order: 

\textit{  LLM $X$ rates lesson $i$ as better than lesson $j$: $x_i>x_j$.  Does this align with human experts $Y$, $y_i > y_j$? Does this align with student learning $Z$ associated with each lesson, $z_i>z_j$?
}

Adding brackets as indicator functions we get: $x_{ij} = [x_j>x_i] - [x_j<x_i] $ and $ y_{ij} = [y_j>y_i] - [y_j<y_i]$
where the alignment between $x$ and $y$ for two lessons is simply the product $x_{ij}y_{ij}$. All pairwise comparisons between lesson ratings from LLM $X$ and some outcome $Y$ are collected into the antisymmetric matrices $X = (x_{ij})$ and $Y=(y_{ij})$, respectively.  Measuring alignment between $X$ and $Y$ is the aggregation of all pairwise directional alignments (is lesson $i$ better than $j$) across all lessons:\footnote{Confidence intervals are computed using the correction from \cite{fieller_tests_1957}} $\tau_{XY} = \langle X, Y \rangle_{\rm F}/\|X\|_{\rm F} \|Y\|_{\rm F}$
where $\langle \cdot, \cdot \rangle_{\rm F}$ and $\|\cdot\|_{\rm F}$ are the Frobenius inner product and Frobenius norm, respectively.  This formulation mathematically mirrors the needs of the alignment question, and does so by reducing the sensitivity to noise in the ratings. This scale-independent approach allows us to establish meaningful real-world baselines, such as using a teacher's years of experience or a prior VAM as the ratings (Fig. \ref{fig:alignment}). 

\subsection{Decomposing error variance}
\label{sec:vardecomp}

Aggregate misalignment metrics (correlations, mean error) do not reveal \emph{why} LLMs fail: the same overall error can be produced by (i) fixable implementation choices (model selection, prompting), (ii) shared, systemic biases that persist across models, or (iii) transcript- and construct-specific difficulties. To localize failure modes, we structurally decompose the observed misalignment error into fully crossed variance components associated with each facet of our evaluation design.

Let $c\in\mathcal{C}$ index classroom transcript segments (observations), $i\in\mathcal{I}$ rubric items, $m\in\mathcal{M}$ foundation models, and $p\in\mathcal{P}$ prompt families. For each cell $(c,i,m,p)$ we observe a standardized LLM score $\widetilde{X}_{cimp}$ and an aligned (pre-standardized) value-added outcome $\widetilde{Y}_{c}$. The (unsigned) misalignment error is the squared difference
\begin{equation}\label{eq:sqerr_main}
  \hat{e}_{cimp} \;=\; (\widetilde{X}_{cimp} - \widetilde{Y}_c)^2 .
\end{equation}
We then fit a fully crossed random-effects model (Generalizability Theory; \citealp{brennan_generalizability_2001}) that partitions $\hat{e}_{cimp}$ into main effects and interactions among $\{\textsc{Obs}=c,\textsc{Item}=i,\textsc{LLM}=m,\textsc{Prompt}=p\}$:
\begin{align}\label{eq:vardecomp_main}
  \hat{e}_{cimp} \;=\;\; & \mu
  + \!\!\!\!\sum_{\emptyset\neq \alpha \subseteq \{c,i,m,p\},\,|\alpha|\le 3}\!\!\!\!\!\! \nu_{\alpha}
  + \eta_{cimp},
\end{align}
where each $\nu_{\alpha}\sim\mathcal{N}(0,\sigma^2_{\alpha})$ is a mean-zero random effect for facet set $\alpha$, and $\eta_{cimp}\sim\mathcal{N}(0,\sigma^2_{\eta})$ is the cell-specific remainder comprised of the four-way interaction confounded with residual noise. 

For each component $k$ we report its \emph{variance share} $\pi_k=\sigma^2_k/\sigma^2_{\mathrm{tot}}$ (Eq.~\ref{eq:varshare}), which quantifies how much of the misalignment landscape is attributable to (a) developer-controlled levers (\textsc{LLM}, \textsc{Prompt}, and their interactions) versus (b) evidence- and construct-conditioned effects (\textsc{Obs}, \textsc{Item}, and interactions). Appendix~\ref{apx:vardecomp} provides the full model specification, estimation details, code, sign-preserving variants, and a decision-study analysis for how many models/prompts are needed to stably recover the shared misalignment signal.

\section{Data}
\label{sec:data}
All data used in this study originate from publicly available sources, ensuring the reproducibility of our findings. The core dataset is from the National Center for Teacher Effectiveness (NCTE) Main Study \cite{kane_national_2015}, a landmark project that collected extensive data on teaching and learning over three years. The NCTE dataset comprises observations of roughly 350 4th and 5th-grade mathematics teachers across four U.S. school districts. It is one of two educational datasets that contains measures of teaching practice, authentic classroom artifacts, and student learning outcomes at scale (VAMs). 

\subsection{Classrooms and downstream task observation instruments} \label{sec:rubrics}
Our primary input for the LLMs are anonymized transcripts \cite{demszky_ncte_2022} of video-recorded classroom lessons using the test set defined by \cite{wang_is_2023}. Human raters \citep{kane_national_2015} rated lessons by watching the videos. These same lessons were previously evaluated by teams of expert human raters using two validated, multi-dimensional observation instruments currently used in the field: \textbf{Mathematical Quality of Instruction (MQI):} A content-specific framework for evaluating the richness and precision of mathematics instruction \cite{hill_mathematical_2008} and \textbf{Classroom Assessment Scoring System (CLASS):} A framework for assessing general dimensions of classroom quality, including behavior management and class climate \cite{pianta_classroom_2008}. The 63 MQI raters were recruited for their mathematics instruction expertise and underwent rigorous certification and continual calibration to ensure scoring reliability, a practice also used with the 19 CLASS raters \cite{blazar_attending_2017,kane_national_2015}. The expert scores from these instruments serve as our first target for LLM alignment.\footnote{Human rater data:\url{https://www.icpsr.umich.edu/web/ICPSR/studies/36095/datadocumentation}; Transcripts are available for 1,600 of the lessons \cite{demszky_ncte_2022} \url{https://github.com/ddemszky/classroom-transcript-analysis}; replication test set and prompts \url{https://github.com/rosewang2008/zero-shot-teacher-feedback}. Full replication LLM output for this study can be found \url{https://drive.google.com/file/d/1fP7xyKasJ4Ui6di-SlY3TLdNQCeg59Tr/view?usp=sharing}}

\subsection{VAMs: Value-added to student learning}\label{sec:vams}
To connect model outputs to the intended impact of teaching, we use value-added measures (VAMs) of student learning. VAMs are widely considered the gold standard for statistically estimating a teacher's causal effect on student achievement gains \cite{bacher-hicks_evaluation_2017,bacher-hicks_experimental_2019,kane_gathering_2012}. A VAM quantifies how much a teacher's students grew academically over a school year compared to their expected growth, controlling for prior achievement, context, peer effects, and other student-level covariates. The NCTE dataset provides multiple high-quality VAM scores for each teacher. Following established practice \cite{kane_gathering_2012}, we use stacked VAMs (Apdx. \ref{apx:disatt}, \citealp{kane_national_2015}) for each teacher-year corresponding to the observed lesson estimate of a teacher's contribution to student learning. Crucially, our evaluations assume that improved teaching practices are positively associated with improved gains to student learning, in aggregate and on average, even if all sources of noise cannot be removed. To our knowledge, this is the first study to use VAMs as a benchmark for evaluating generative LLMs.

\section{Results and Discussion}

\subsection{The convergent bias of foundation models}\label{sec:discdcor}
A noteworthy result is the striking \emph{behavioral homogeneity} of LLMs when evaluating classroom transcripts. As summarized in Figure~\ref{fig:dCor_summary}, different models' ratings are substantially more correlated with one another than with expert human ratings, both within the same task and across different instructional tasks. 

\begin{figure*}[th]
    \centering
    \includegraphics[width=0.95\linewidth]{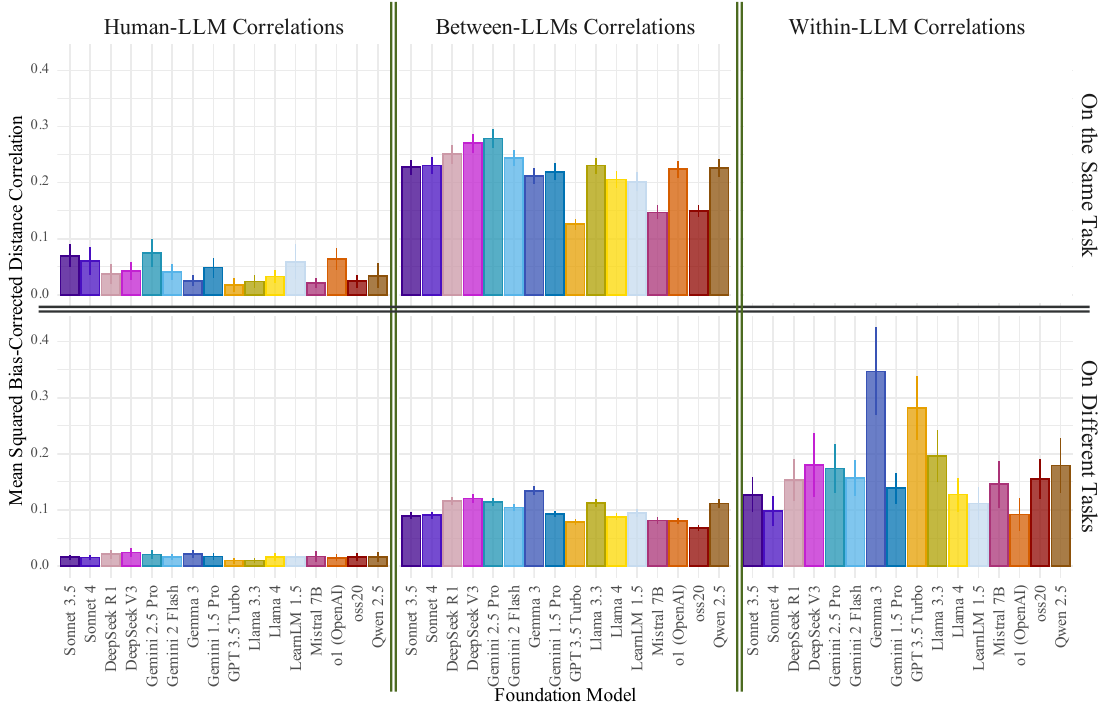}
    \caption{\textbf{Mean Inter-task Bias Corrected Squared Distance Correlations $\operatorname{dCor}^2_n$ }: between LLMs and human raters across different evaluation tasks. \textbf{Top row: Same-task Correlation} Mean inter-rater distance correlations across transcripts for the same task and (\textbf{bottom row: different task correlation}) for different tasks using the same transcript. (\textbf{left: correlations with humans}) Mean inter-rater distance correlations with expert human raters, (\textbf{center: correlations with other LLMs}) with other LLMs, and (\textbf{right: intramodel intertask correlations}) each LLM with itself. The top right is omitted for redundancy.  Lines are standard errors for each estimated mean. Means and SEs were computed under Fisher's $z$ transformation and back transformed to preserve variance. The complete correlation matrix for each task and model are found in Figs. \ref{fig:distance} and \ref{fig:distance_class}}
    \label{fig:dCor_summary}
\end{figure*}

Two patterns are especially notable. First, \textbf{LLM-LLM agreement is consistently higher than LLM-human agreement}. Second, \textbf{within- and between-model inter-task correlations are high}, indicating that a model's outputs for distinct instructional constructs (e.g., language support vs. remediation) tend to move together more than expert ratings do. In other words, when confronted with authentic classroom discourse, models appear to rely on a shared latent heuristic of ``good teaching'' that is not strongly anchored to the constructs that human observers are trained to distinguish.

This convergence is plausibly explained by what these systems share: an autoregressive pretraining objective and large-scale Internet text. Authentic elementary classroom discourse is largely absent from such corpora, forcing generalization under distribution shift. The resulting shared bias echoes emerging evidence that, as models scale, their representations and judgments can become increasingly correlated across developers and architectures \citep{kim_correlated_2025,ren_human-ai_2024,huang_values_2025}. The following section investigates whether these model convergences are aligned.

\subsection{Perils of proxy alignment}\label{sec:discalign}
Figure~\ref{fig:alignment} juxtaposes two forms of alignment for each model and task: pairwise concordance correlations with expert ratings ($\tau_{S_f X}$, x-axis) and with student learning gains ($\tau_{S_f Y}$, y-axis). The central empirical finding is a \textbf{systematic disconnect} between these axes. Models that appear more aligned to expert judgments are \emph{not} correspondingly aligned to learning, and in many cases, are \emph{more negatively} associated with learning outcomes. The magnitudes of the baselines are consistent with the literature \citep{mccaffrey_intertemporal_2009,kane_estimating_2008,kane_gathering_2012} (see Appendix \ref{app:priorvam}).

\begin{figure*}
    \centering
    \includegraphics[width=1\textwidth]{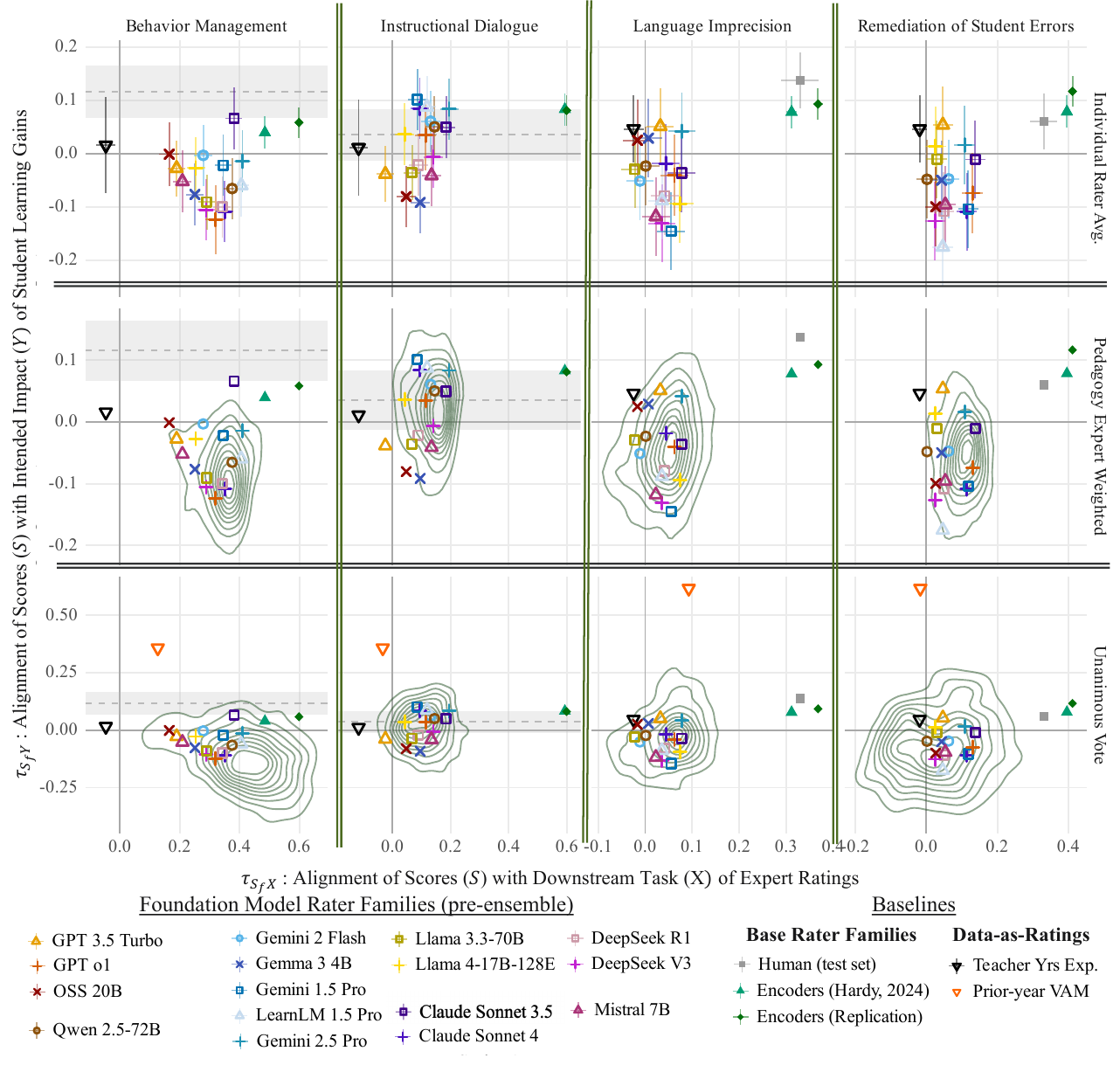}
    \caption{\textbf{(Mis)alignment with Downstream Task (Teaching) and Intended Impact (Learning)}: The \textbf{x-axes} measure the alignment of scores ($S_f$) from each LLM, ensemble, and baseline $f$ with expert human ratings on downstream tasks ($X$) on the quality of teaching skills in a given lesson: $\tau_{S_fX}$. Similarly, the \textbf{y-axes} measure alignment with the value-added to learning via student achievement gains ($Y$): $\tau_{S_fY}$. Each color-shape combination represents a different rater family or baseline. Each column represents a specific instructional rating task, from (\textbf{left} to \textbf{right}): \texttt{\texttt{CLBM}, \texttt{CLINSTD}, \texttt{LANGIMP}}, and \texttt{\texttt{REMED}}. Each row represents a distinct implementation scenario: \textbf{Top row (individual rating models)} The 95\% CI are shown for individual models. \textbf{Middle row (Pedagogy Expertise Weighted Ensembles)} and \textbf{Bottom row (Unanimous Vote Ensembles)}: The alignment correlations of all possible model-prompt ensembles are represented by the density contours. The innermost contours represent the alignment regions' greatest ensemble performance density. For the bottom row ensembles, we amplify the shared signal by only measuring observations where all three models are in agreement (See Appendix \ref{sec:ensemble-failure}).  \textbf{Baselines}: the gray line and shaded regions on \texttt{CLBM} and \texttt{CLINSTD} represent the estimate and confidence intervals from the human expert rating alignment with student learning gains on the study's joint test set. The green Encoder models (data from \cite{hardy_all_2024} and our own replications) are LMs but not LLMs, and they merely serve as a deep learning baseline for possible alignment based on transcripts. The \textbf{baseline of Teacher Experience} represent pairwise comparisons that always rank more experienced teachers higher than less experiened ones. For better plotting details, the ``oracle'' \textbf{VAM baseline} which puts a higher value to teachers with higher prior year VAMs,  is only displayed on the bottom row. Random baselines (both uniform and stratified) predictably clustered around the origin and, with no real-world example of this as a valid baseline, we exclude them from the plots.}
    \label{fig:alignment}
\end{figure*}

This pattern constitutes a particularly consequential failure mode: \textbf{proxy alignment without impact alignment}. A system can appear to ``do the job''—produce plausible, rubric-concordant scores—while selecting classrooms that are \emph{worse} on the objective schooling ultimately values. This mirrors the caution from scientific-discovery evaluations: benchmark success can obscure fragility in the behaviors that matter when evidence is messy and objectives are implicit rather than explicitly labeled \citep{song_evaluating_2025,zhou_navigating_2023,zhou_relying_2024}.

We also observe the converse: some smaller or older models occasionally show slightly better $\tau_{S_f Y}$ while exhibiting weaker $\tau_{S_f X}$. Qualitative inspection of failures suggests these cases often reflect \emph{task noncompliance}: rather than applying the intended rubric, models may latch onto superficial transcript features that, coincidentally, correlate with higher achievement gains in this sample. The implication is that neither ``sounding pedagogical'' nor matching expert observation scores is sufficient evidence that the model has learned a transferable representation of effective instruction.

We analyze whether additional test-time ``reasoning'' variants yield improvements analogous to those experienced in many complex tasks. Comparing paired models sharing the same base (e.g., DeepSeek-R1 vs. DeepSeek-V3.1; and similarly GPT-3.5-class vs. reasoning-oriented variants), we find \textbf{no measurable improvement} on either axis of alignment in Figure~\ref{fig:alignment}. Findings comparing the chain-of-thought prompt were similar. For classroom evaluation, additional reasoning context window alone does not appear to repair the core mismatch between model judgments and constructs that predict learning gains.

\subsection{Ensembling exacerbates misalignment}\label{sec:discensem}
A natural response to noisy model behavior is ensembling. We evaluate two conceptually opposite approaches: (i) a \emph{pedagogy-expertise-weighted} ensemble, where model votes are weighted by pedagogical benchmark performance, highlighting LLM uniqueness, and (ii) a \emph{unanimous-vote} ensemble, where we score only those instances where all models agree, emphasizing shared signal (\S\ref{sec:ensemble-failure}).

Neither strategy improves alignment with student learning. Instead, both frequently \textbf{worsen} $\tau_{S_f Y}$, especially on core instructional dimensions such as remediation of student errors and behavior management (Figure~\ref{fig:alignment}). This result has two implications. First, it suggests that benchmark-measured pedagogical ``knowledge'' does not translate into reliable recognition of effective pedagogy in authentic discourse (i.e., the benchmark construct is not externally valid for this downstream setting). Second, it indicates that when models agree, they may be amplifying a shared but flawed heuristic; consensus is not evidence of correctness with correlated errors \citep{chen_are_2024,zhu_bare_2025}.

\begin{table}
\centering
\caption{\textbf{Misalignment Shares}: Proportion of Variation in Squared Error by Source. Row partitions separate main, second-order, and higher-order effect posterior distributions. }\label{tab:sq_err_decomp}
\resizebox{\ifdim\width>\linewidth\linewidth\else\width\fi}{!}{
\begin{tabular}{lcccc}
\toprule
Facet of Error Variation & median & MAD & MAP & 95\% HDI \\
\midrule
ITEM & 0.08 & 0.08 & 0.03 & [0.01,0.62] \\
LLM & 0.07 & 0.04 & 0.06 & [0.02,0.18] \\
OBS & 0.02 & 0.01 & 0.02 & [0,0.03] \\
PROMPT & 0.02 & 0.03 & 0.00 & [0,0.58] \\
\midrule

ITEM:OBS & 0.04 & 0.01 & 0.05 & [0.01,0.06] \\
LLM:ITEM & 0.03 & 0.01 & 0.03 & [0.01,0.06] \\
LLM:OBS & 0.00 & 0.00 & 0.00 & [0,0] \\
LLM:PROMPT & 0.02 & 0.01 & 0.01 & [0,0.04] \\
PROMPT:ITEM & 0.01 & 0.01 & 0.00 & [0,0.05] \\
PROMPT:OBS & 0.00 & 0.00 & 0.00 & [0,0] \\
\midrule
LLM:ITEM:OBS & 0.19 & 0.03 & 0.21 & [0.05,0.23] \\
LLM:PROMPT:ITEM & 0.03 & 0.01 & 0.03 & [0.01,0.05] \\
LLM:PROMPT:OBS & 0.14 & 0.02 & 0.15 & [0.04,0.17] \\
PROMPT:ITEM:OBS & 0.02 & 0.00 & 0.02 & [0.01,0.03] \\
LLM:PROMPT:ITEM:OBS +  $\epsilon$ & 0.24 & 0.04 & 0.26 & [0.06,0.29] \\
\bottomrule
\end{tabular}}
\end{table}

\subsection{Persistent artifacts of autoregression}\label{sec:discvardec}

\paragraph{What the decomposition tells us (and why we need it).}
Aggregate metrics (e.g., mean error, correlation) cannot distinguish \emph{fixable} misalignment from \emph{structural} misalignment. Our fully crossed variance decomposition over \textsc{Obs} ($c$), \textsc{Item} ($i$), \textsc{LLM} ($m$), and \textsc{Prompt} ($p$) localizes where misalignment lives: in developer-accessible choices (\textsc{LLM}, \textsc{Prompt}), in the evidence and construct being scored (\textsc{Obs}, \textsc{Item}), or in their interactions. This turns a vague diagnosis (“the model disagrees with VAM”) into an actionable one (“error is dominated by transcript-conditioned interactions, so model shopping and prompt iteration will not reliably fix it”).

\paragraph{Model choice and prompt choice are weak levers.}
Table~\ref{tab:sq_err_decomp} shows that the main effects of \textsc{LLM} and \textsc{Prompt} explain only a small portion of the squared-error landscape (median shares $0.07$ and $0.02$), and even their interaction remains small (median $0.02$). If misalignment were caused by a few deficient models or a single poor prompting recipe, these components would be large and stable. Instead, the decomposition implies a more sobering practical conclusion: swapping models and prompts may change outcomes locally, but it is unlikely to yield a \emph{reliable} improvement in validity against intended impact.

\paragraph{Prompting is brittle, not corrective.}
Although the typical prompt effect is small, its posterior is long-tailed (wide HDI despite a low median), indicating \emph{prompt brittleness}: prompts can occasionally inject large error without delivering consistent gains across transcripts. The risk is asymmetric: prompt changes can produce dramatic “wins” on a handful of cases while silently degrading performance elsewhere, an often missed failure mode. 

\paragraph{Misalignment concentrates in context-conditioned interactions.}
The dominant variance shares occur in higher-order interactions that condition on the classroom text, especially \textsc{LLM}$\times$\textsc{Item}$\times$\textsc{Obs} ($0.19$) and \textsc{LLM}$\times$\textsc{Prompt}$\times$\textsc{Obs} ($0.14$), together with a substantial cell-specific remainder ($0.24$). This is the signature of a transcript-conditional failure mode: models behave unpredictably on particular kinds of instructional evidence. Scientifically, that pattern is consistent with LLMs relying on latent proxies (fluency, affect, participation cues, stylistic norms) that are only loosely coupled to student learning gains and whose influence varies with the segment.
\paragraph{The direction of error is more shared than the magnitude.}
Squared error $\hat{e}$ hides whether models over- or under-estimate impact. When we preserve direction using the signed quadratic error $\hat{e}^{\pm}$ (Appendix~\ref{apx:signpreserve}), the \textsc{Item}$\times$\textsc{Obs} share increases substantially (Table~\ref{tab:app_sgn_sq_decomp_full}). This indicates that models often drift in the \emph{same direction} on the same item--segment pair, even when they disagree stochastically about how large the error will be. In other words, LLMs may look noisy at the surface, yet still share a coherent (and undesirable) inductive bias about what “good teaching” looks like in text.

\paragraph{Implication: validity fails in regimes where users most need reliability.}
Taken together, these findings explain why surface-level interventions underperform. The misalignment is largely not a property of a single model or prompt; it is an evidence-conditioned artifact that persists across models and is revealed only when we ask about intended impact. For high-stakes applications, this motivates a different optimization target: reducing the \emph{shared} misalignment component (the part that survives averaging over models/prompts), rather than chasing occasional prompt-dependent gains on fixed transcript sets.

\section{Conclusion}\label{sec:conc}

This paper studies a common but under-instrumented problem in contemporary NLP: LLMs can exhibit strong internal agreement and high apparent competence while failing to align with the outcomes a domain actually values. Here, we evaluate LLM-based scoring of classroom instruction against two external criteria: expert human observations (the downstream task) and value-added measures (VAM) of student learning gains (the intended impact). Across leading models and prompting strategies, we find that (i) models converge behaviorally, (ii) their convergence is misaligned with intended impact, as alignment with the downstream task is not a strong enough proxy,  and (iii) common ensembling and prompt engineering do not reliably repair the gap .

A key methodological contribution is a structural decomposition that makes these failures measurable in high-noise regimes. Rather than treating misalignment as a scalar, we model the observed error as arising from multiple, simultaneously operating facets—\textsc{Item}, \textsc{Obs}, \textsc{LLM}, and \textsc{Prompt}—and their interactions in a fully crossed design. This variance decomposition changes what an evaluation can claim. It supports a practitioner-relevant statement of the form: \emph{how much of the observed error is plausibly removable by changing implementation choices (model/prompt), and how much persists as transcript-conditional, shared behavior that will remain even if one “shops” for a better model?} In our data, the dominant error structure concentrates in higher-order interactions involving the classroom evidence, with comparatively small and brittle contributions from prompt and model main effects. The implication is concrete: when misalignment is concentrated in evidence-conditioned interactions rather than in stable model or prompt effects, iterative prompt engineering and model substitution are unlikely to yield reliable validity with respect to intended outcomes.

More broadly, the decomposition offers an evaluation template for other AI and NLP problems where (a) labels are noisy, (b) the unit of prediction mismatches the unit of impact, or (c) the true target is only observable through delayed, aggregated, or confounded measurements. Examples include human-centered summarization (where user decisions matter more than ROUGE), clinical decision support (where outcomes may occur weeks or months later), content moderation (where harms are downstream and partially unobserved), and policy or educational tools (where success is causal and context-dependent). In such domains, chasing marginal benchmark gains can be rational while still missing the central scientific question: \emph{does the system improve the world it is deployed in?}

We therefore argue for a shift in evaluation practice. First, NLP evaluation should more often include an explicit intended-impact target—however noisy—and should report uncertainty and sensitivity rather than only point estimates. Second, when noise is unavoidable, the correct response is not to abandon measurement, but to use designs and estimands that are robust to it: multifacet decompositions, decision studies, hierarchical rater models, residualization and mediation analyses, and pre-registered stress tests. Third, the community should treat “high noise” not as a disqualifier for rigorous work, but as a signal that the domain is real: many high-stakes applications are difficult precisely because outcomes are multi-causal, delayed, and imperfectly measured. We need to listen more closely to measure the noise.

Finally, our results suggest a cautionary principle for applied LLM evaluation: \emph{knowledge without wisdom is detectable}. When models exhibit strong consensus yet their shared error correlates poorly or negatively with intended impact, scaling and superficial alignment interventions are unlikely to suffice. Progress will require methods and training signals that explicitly bind model judgments to causal, downstream consequences, and evaluation frameworks that can reveal when apparent competence fails to translate into real-world benefit. In short, LLM evaluation should demonstrate the wisdom that our very smart LLMs may lack.

\section*{Limitations}\label{limit}
Our central aim is methodological: to measure alignment between LLM-generated evaluative scores and both expert ratings (a downstream task target) and student learning gains (an intended-impact target) under realistic noise. This design necessarily inherits the constraints of education measurement. Value-added measures (VAM) are imperfect, high-variance estimates of causal impact; transcript segments are partial, lossy views of instruction; and expert rubrics, while informative, exhibit nontrivial rater disagreement. These limitations do not invalidate the study---they define the regime we seek to operate in---but they bound the claims we can make. Accordingly, we interpret variance decomposition results as statements about \emph{where observed misalignment concentrates} under this measurement system, not as a definitive census of all sources of pedagogical effectiveness. We also emphasize that our estimates are conditional on the sampled items, segments, models, and prompt families; extending the universe of items, observation windows, or prompting mechanisms could change component magnitudes, even if the core inference about controllable versus systemic error structure persists.

In the presence of low reliabilities observed by human annotators, we echo that there is still substantial and meaningful information, even with the measurement error found throughout education contexts \citep{ho_reliability_2013,mccaffrey_uncovering_2015, hardy_all_2024,hardy_measuring_2025}.  
The test set used \citep{wang_is_2023} may have unobserved confounding factors in its construction. To account for this, we also investigated other methods of estimating these effects given the complexity of the relationships. In the extended Appendix \ref{apx:altmethods}, we demonstrate a multi-stage residualization of confounding and mediating effects. We are pleased to report that our conclusions of this paper are robust to and even are strengthened by these additional tests. For additional information, context, and tests for working with high noise data, see also Appendices \ref{apx:vam} and \ref{sec:highnoise} and \cite{brennan_advanced_2001, mccaffrey_intertemporal_2009, brennan_generalizability_2013,hardy_all_2024,casabianca_psychometrics_2025}.

The transcript data used in this work contain only fourth- and fifth-grade mathematics classrooms from the United States. Furthermore, the associated ratings pertain solely to a subset of rating items on a specific rubric, which may introduce limitations when addressing other tasks of classroom instructional support for children. While there is no evidence to suggest that findings would be different in other primary classrooms, the data make generalization to all classrooms not demonstrable in the current study.

Meaningful representation of student classroom learning is absent on the internet largely for laws protecting children’s privacy. The text for our data was only anonymized and made public in 2022 (\citeauthor{demszky_ncte_2022}), and it is not ``crawlable''; nor is it easy to link these data with the (not crawlable) annotations and value-added measures of student learning (VAMs) \citep{kane_national_2015}. The linking of these two sources makes it 1 of 2 extant datasets having classroom interactions with both expert ratings on real teaching instruments and VAMs. The second dataset (``MET Project'') is not publicly available. Part of our interest in doing this research is that conducting these studies is extremely expensive, and economic drivers exclude this line of work from studies that would measure impact (see the exclusion of education, social work, therapy, and other fields from studies like, for instance, \citealt{patwardhan_gdpval_2025}). 

One purpose of this study is to measure the extent to which LLMs have the capacity to conduct tasks in downstream applications \textbf{\emph{responding to classroom teaching and learning}}, where these challenges not only are particularly prominent but also bear significant real-world consequences. Authentic educational content, particularly for school-age children, is effectively non-existent on the Internet and therefore absent in the pretraining data of large LLMs.\footnote{Part of this underrepresentation is a result of the protection of data and educational records of minors} Studies about the quality of teaching and learning are expensive \citep{grissom_effective_2013,liu_measuring_2021,jurenka_towards_2024}, with only two major studies having measures of both teaching and learning: the MET study \citep{kane_have_2013,kane_gathering_2012} and the NCTE Main Study \citep{kane_national_2015}, the latter of which is the source of data for this study. Even then, no transcripts or artifacts of classroom discourse are meaningfully annotated with respect to reliable measure of a) the \textbf{\textit{quality of teaching}}, using authentic instruments designed for humans or b) the \textbf{\textit{ quality of learning}}. In fact, the majority of instructional materials that would be available on the Internet for pre-training are of low quality \citep{polikoff_supplemental_2019, northern_dear_2019, edreports_state_2023,tntp_opportunity_nodate, tntp_opportunity_2024}. 

Many critical education tasks are poorly represented in training data for GPT models. Authentic natural language found in K12 public education contexts is rare on the internet, especially data with meaningful labels or interpretations, because these data involve children and typically have more stringent protections (e.g. FERPA and COPPA in the United States). Thus, a limitation noted by this paper is the need to create more datasets, a limitation shared with many fields, including LLM-as-scientist studies \citep{song_evaluating_2025,alampara_probing_2025,mirza_framework_2025}. 


\section*{Ethical considerations}\label{sec:eth}
This paper is about measurement: how to detect, quantify, and localize misalignment when LLM scores are used as stand-ins for human judgment and, more importantly, for downstream outcomes. But our case study is not measurement in the abstract. Classroom evaluation is a high-stakes domain in which errors are not evenly distributed: mismeasurement can change students' learning opportunities, shape teachers' careers, and amplify existing inequities. The purpose of this section is therefore twofold. First, it explains why the empirical patterns we uncover---strong inter-model consensus, weak connection to intended impact, and limited efficacy of prompt/model tweaks---are ethically consequential in K--12 settings. Second, it provides a roadmap for what “responsible next steps” look like when the technical work reveals that validity cannot be assumed.

Our findings shift the ethical question from ``Is an LLM accurate enough?'' to ``Accurate \emph{for what}, and by what evidence?'' In education technology, it is tempting to treat alignment with expert rubrics as sufficient validation. Our results warn against that shortcut: systems can look reasonable to adults (and even agree with each other) while remaining misaligned with student learning gains. This mismatch creates a distinctive risk profile: deployment can produce confident, scalable judgments with unclear or negative educational benefit. The ethical challenge is thus not merely model bias in the usual demographic sense, but \emph{validity risk}: using an apparently coherent evaluator whose shared inductive biases reward proxies for quality rather than the outcomes schools are responsible for delivering.

For readers asking what to do with these education-specific implications, the remainder of this section is organized around “next-step” operational problems and workable frameworks. We first discuss how to understand educational needs and the limits of what transcript-based scoring can justify, especially for children. We then address equity and accountability: why equal access to an AI tool is not evidence of equal benefit, and what monitoring is required before scaling. Finally, we outline practical safeguards for researchers and developers---including intended-impact evaluation, pre-deployment decision studies, and harm-aware piloting---that translate the measurement methods in this paper into responsible practice in public education contexts.


\subsection{Understanding educational needs of children}
Foundation models’ generality and adaptability as next-token predictors \citep{malach_auto-regressive_2024} have energized the education technology (edtech) sector. We, like many developers, are eager for a world where all children get access to high quality education. With increased hype, why is there yet very little evidence of these models meaningfully improving student learning in K12 contexts? Recent work has even shown that generative AI deployed as conversational agents can harm student learning \citep{bastani_generative_2024,nie_gpt_2024,shein_impact_2024}. However, developers of education technology (edtech) have not been deterred \cite{kornell_edtech_2024,odonnell_heres_2024}, further increasing the high adoption rates of generative AI in educators \cite{bonney_tracking_2024,humlum_adoption_2024,bick_rapid_2024}. So we feel the need to elaborate on the pertinent ethical considerations for edtech in public education spaces.

Answering whether a LLM is good enough or better than the alternative is difficult \citep{damour_underspecification_2020,hutchinson_evaluation_2022}. Evaluating model pedagogical outputs in K12 education is even harder than for most other disciplines \citep{denny_generative_2024,macina_opportunities_2023,wollny_are_2021}, where understanding the impact is paramount. 
One challenge in evaluating for K12 impact is the paucity of datasets that allow for quantitative evaluation of model performance with LLMs underperforming \citep{jurenka_towards_2024,tack_bea_2023,ormerod_automated_2024}.  Unfortunately, poor model performance on existing evaluable K12 datasets is often ignored and can be hidden by changing the nature of the evaluation \citep{rottger_political_2024,schaeffer_are_2023}. For example, instead of asking the model to accurately autoscore student work—a critical task for K12 impact—one could ask it to write plausible feedback which might then be evaluated indirectly using surveys of relative preference by a few raters.

Failure of any educational resource is generally not a binary characteristic, but rather measures of extent along \textit{continua of quality}. For a given group of children, an excellent human tutor would facilitate each child learning at their fullest capacity, however defined, whereas a weaker tutor may only be able to help some students learn some of the content some of the time, much like some non-humans \citep{collins_building_2024,pardos_learning_2023,vaccaro_when_2024}. If the human tutor were replaced by a textbook, there would still a very small subset of children who could fully benefit from that intervention, but most students would not be as well served. These continua of quality exist for all educational resources, e.g., lesson plans, remediation activities, curricula, formative assessments, systems for classroom climate, IEPs, feedback to students, etc.
\paragraph{The Paradox of Free Advice} Similarly, all AI-generated content will be imperfect, and the degree of quality available can be difficult to recognize. Freely available generative AI introduces a challenge that we will call the \textbf{Paradox of Free Advice}: 

\textit{Those needing more guidance are also those that are less discerning of the quality of guidance or support offered}. 

Such individuals may turn to generative AI tools, which can be speciously compelling and confident, even when such models are not deserving of our trust \citep{kim_im_2024,wu_style_2023,yan_refutebench_2024,zhou_relying_2024}. The implications are that, for example, while an automated service may save a teacher time, it may result in a child losing learning by spending time or resources in less efficacious interventions \citep{acemoglu_tasks_2022,holmes_artificial_2022}. In fact, some practices in Reinforcement Learning from Human Feedback (RLHF) \citep{christiano_deep_2023} exacerbate this even further: RLHF-based model tuning simultaneously leads to less accurate outputs and yet increased human confidence that the less accurate output is trustworthy \citep{wen_language_2024}. If even the most expert humans have shown worryingly bad tendencies in collaborative decision making with AI \citep{agarwal_combining_2023}, how will children and public school teachers fare?
The \textit{Paradox of Free Advice} can affect researchers and scientists as well. Failing to have access to or invest in high quality evaluations utilizing true expertise \citep{hosking_human_2024} for the expected human judgments can result in evaluation and reporting of research findings that deepen gaps in quality when convenience samples of human experts rely on these same systems to respond to researcher requests \citep{veselovsky_artificial_2023}.

Biased performance is even harder for individual users to see on single use cases, since GPT models can often appear reasonable and trustworthy even when they are less accurate \citep{klingbeil_trust_2024,wen_language_2024,zhou_navigating_2023}. Children, like adults, become more trusting of these tools with exposure \citep{kosoy_childrens_2024}, an observation that demands deeper exploration into biases, as the models have already been associated with behavioral changes in postsecondary learning contexts \citep{abbas_is_2024,nie_gpt_2024,zhai_effects_2024}. 

\paragraph{Inequity along Continua of Quality} Educational tools targeted for use in compulsory public schooling contexts have the intrinsic responsibility to provide equitable learning for all children (and to clearly communicate otherwise if that cannot be demonstrated). For this paper, \textbf{equality} is defined as equal access to resources, opportunities, tools, or systems, and, in contrast, \textbf{equity} is defined as equal access to the \textit{benefits} of those resources, opportunities, tools, and systems. In other words, merely offering all students access to a particular educational resource is not sufficient evidence of equitable outcomes for students, and, in many cases, may exacerbate existing gaps \citep{benjamin_race_2019,crooks_access_2024,eubanks_automating_2019,kasneci_chatgpt_2023,madaio_beyond_2021}. 
In general, over the last couple decades, edtech has not provided evidence that it can equitably improve student outcomes in K12 public education and, in many cases, has even produced results which have widened preexisting inequities in learning \citep{hansen_democratizing_2015,reich_failure_2020,reich_good_2017}. This increased inequity often arises as an example the “Matthew Effect,” a term for the disparity-widening effects of accumulated advantage that make it easier to benefit further from new advantage. This phenomenon appears in many education contexts \citep{bahr_double_2007,kempe_are_2011,reich_failure_2020,stanovich_matthew_1986}. As a pre-LLM edtech example, asynchronous educational content learning experiences, such as those found in massive open online courses (MOOCs), used to support high schoolers recovering credits towards graduation have widened gaps for those most in need of support \citep{heinrich_look_2019}: those who needed more help who, in fact, got less help from the intervention. As with the example of tutors of varying quality, poorly executed K12 edtech solutions can exacerbate such issues, and thus an even greater need to be able to identify tasks where evaluating output quality is more challenging or uncertain. 
\paragraph{Principle of Precision in education} Equitable and individualized K12 learning requires precision. Recognizing that most uses of generated K12 content, often branded as saving teachers time, are, in fact, \textit{making instructional decisions}—choices about what and how to teach—not just generating ideas for top performers. 
Imprecision in instructional decisions degrades its quality, resulting in inefficient teaching and unforced losses to a child’s learning time. A critical component for helping learners “catch up” is to make instructional decisions that maximize time spent on things students need with precision, rather than on content that they already know. AI use can help with writing when precision is not critical, but has led to worsened decision-making outcomes \citep{vaccaro_when_2024} and can further marginalize learners whose needs are not well represented by the mean of the distribution \citep{treviranus_learning_2022}.

Precise instructional decisions are needed for efficiently and effectively resolving student misconceptions, as a practical example. Using an imperfect but illustrative analogy from computer programming, we could consider confusion and misconceptions as “bugs” in a student’s reasoning \citep{brown_repair_1980, vanlehn_mind_1990}.  Both fixing and not creating new bugs require expert precision, which may not be a strength of the current GPT models when supporting struggling coders. 


\subsection{Demands placed on educators}
Public K12 educators have deep insight into their specific students and their idiosyncratic teaching tendencies. But critically, the average K12 teacher is not an education expert at the level needed and assumed by researchers and developers for many edtech questions, resources, and products. This is in no way saying that researchers and developers should not consult educators. On the contrary, K12 educators know things about the day-to-day dynamics within their classroom that researchers, developers, and designers can learn from \citep{kizilcec_advance_2024}. This hyperlocalized expertise about the children in their custody may explain why most parents feel like their schools are headed in the right direction while most think that overall schooling is not headed in the right direction \citep{brenan_k-12_2021,horowitz_parents_2022,saad_americans_2022}. However, hyperlocalized educator expertise does not mean that: 
\begin{enumerate}[nosep]
    \item the educator is effective and that their teaching practices result in significant learning gains for students;
    \item the educator’s best practices and insights generalize to other (a) students, (b) classrooms, (c) teachers, (d) content areas, (e) grades/ages, (f) schools, or (g) contexts;
    \item the task or idea the educator is being asked to do, evaluate, or provide feedback on is a task about which they are expert and about which they can accurately identify the conditions needed for its generalization; nor
    \item the educator is aware of the extent to which they have the expertise described here.
\end{enumerate}
These criteria of expected expertise can easily be extended by replacing “educator” in the above with “researcher” or “technologist”. Finding qualified subject matter experts is critical to high-quality solutions \citep{zhou_how_2023}. As important as working with and understanding the needs of K12 educators is for finding solutions, edtech research and products are typically not built for a handful of specific teachers with hyperlocalized expertise, but for education contexts generally. 
\citeauthor{ho_reliability_2013} found that 
school administrators---with more general expertise---were more discerning of differences in teaching quality and produced ratings that were 12\%--25\% more reliable compared to other teachers with similar/different teaching certifications, respectively 
\citep{ho_reliability_2013}.  In education literature this is at least in part because school administrators have more generalizing power from regular exposure to different classrooms. However, even school administrators benefit from localized expertise: Ho and Kane found that administrators from the same school as the instruction being evaluated were 18\% more reliable than non-local administrators.

Expanding the sample of or crowdsourcing across educators may not lead to quality annotations, effective solutions, or generalization across contexts \citep{macina_opportunities_2023}. For some K12 materials and practices, a majority of teachers (and even professors of education) may not be implementing practices that are known in research to be more effective, such as the disparity between known effective practices for early literacy instruction  and common practices in schools and teacher prep programs \citep{foorman_foundational_2016,kurtz_early_2020,rix_new_2023,solari_translational_2020}.

The Principle of Precision applies across continua of quality in education content creation, which represents most current applications of GPT models as educator assistants. While there is clear evidence about the importance of using high quality instructional materials \citep{boser_hidden_2015,edreports_state_2023,kaufman_changes_2018,opfer_implementation_2017,steiner_curriculum_2017,tntp_opportunity_nodate}, most of these materials are either not easily accessible or have a large enough presence online for the purposes of model training. Unfortunately, most of the K12 instructional materials easily accessible online by and popular among educators are not high quality \citep{northern_dear_2019,polikoff_supplemental_2019}. Although automatically generating materials may appear to save teachers time, it may cost students learning time if the ignores the important attributes of instructional materials such as curricular coherence, meeting criteria for “high quality”, and teachers’ abilities to both recognize those criteria and use curricula effectively \citep{chu_fundamental_2021,edreports_review_2021,kane_teaching_2016,polikoff_exploring_2020,short_elements_2020,tntp_opportunity_2024}. Poorer quality materials 
and the \textit{Paradox of Free Advice} would logically lead to the exacerbating Matthew Effect \citep{acemoglu_tasks_2022,capraro_impact_2024}. While techniques for content generation are improving \citep{balepur_expository_2023,rooein_beyond_2024}, the criteria for high quality instructional materials, the coherence in the learning, and the capacity of the teacher to recognize and deliver such content are challenges not yet addressed.

\subsection{Potential solutions to challenges}
\paragraph{Maximizing Education Expertise}
For improving expertise in the field of edtech products, we recommend that researchers and developers acquire the needed expertise with much more intentionality by selecting few experts who, when combined, jointly maximize expertise relevant to the target audiences and contexts of generalization. If improvements to student outcomes are desired, K12 experts should have a demonstrated track record of positively impacting student outcomes, and the generalizable insights they make should correspond with the contexts of their track record. If an expert does not allow for generalization into some content or context of interest, experts should be added. Non-K12 subject matter experts in academia or industry should likewise represent the breadth of the content and should include expertise in the assessment of intended student outcomes. We make the claim that the skills and capacity needed to lead initiatives that are high-quality, impactful to student outcomes, equitable, and scalable are exceedingly rare. This skill set is not found in an average K12 educator, so greater effort must be made by all.

School leaders working with teachers to improve the quality of instruction typically evaluate the teacher's proficiency in a range of competencies (typically measured during in-class observation and evaluation on a teaching rubric; see \citealt{aguilar_developing_2013, bambrick-santoyo_get_2016,bambrick-santoyo_leverage_2018}), then determine which competencies are most important to improve first (i.e., which change will have the biggest impact on student learning), and then provide supportive feedback and coaching. Without strong expertise and measures of student learning, it is challenging for practitioners to prioritize instructional needs and aligned practices from among the many elements of good teaching \citep{saphier_skillful_2008,darling-hammond_what_2014,hammond_culturally_2015,lemov_teach_2015,lemov_teach_2021,liljedahl_building_2021,darling-hammond_implications_2020,schwartz_abcs_2016} and for researchers to empirically quantify the impact of good teaching practices \citep{pianta_conceptualization_2009,charalambous_13_2019,blazar_challenges_2022,jurenka_towards_2024}.

\paragraph{Measuring what matters }
Additionally, we would like to add that perhaps the most important direction for evaluating future edtech work is improving the ability of models or systems to accurately recognize the state of student learning. The field cannot correct what it cannot detect. And, at present, the Paradox of Free Advice applies to developers and researchers in the presence of GPT model speciousness and current practices for K12 expertise. Meaningful, rapid iteration on edtech tools and products cannot be without reliable and valid measures of student learning. 
\paragraph{Dealing with data deficiency }

Distributional pluralistic alignment of models in new tasks is a difficult due to limited knowledge of how to calibrate models to be more representative \citep{sorensen_roadmap_2024}. Research in pretraining data provenance is a promising area of research for addressing biases. For example, studies have shown that performant models can be constructed by curating smaller, more manageable datasets \citep{shi_-context_2024} and then training individual “expert” models from appropriate data which can be combined into a final model \citep{muennighoff_olmoe_2024}. Such techniques have been used to address copyright violations \citep{he_fantastic_2024}. 

However, training on authentic education data, such as transcripts from learning, may not be free of bias or may not support improved model performance toward student learning outcomes without some meaningful interpretation. Being able to identify whether a particular artifact is high-quality is key to improving performance. One possible approach to explore is to collaborate with qualitative researchers who have spent countless hours coding studies \citep{ward_handbook_2020,saldana_coding_2016,mullet_general_2018,philip_articulating_2018} to combine smaller dataset studies. The meaningful curation and combination of such data could serve as a starting point for use of authentic K12 data during pretraining.
\paragraph{Synthetic substitutions}
Even when faced with the challenge of a lack of quality data, we recommend minimizing the use of synthetic K12 educational content. Use of such content would risk perpetuating gaps in equitable learning as models trained thusly increasingly and irreversibly forget the original distributional tails \citep{shumailov_ai_2024,whitney_real_2024,van_der_gun_artificial_2024}. A more expertly built smaller model, even if limited in scope (see for example \citeauthor{hardy_all_2024}), which is more robustly evaluated, could support the improved annotation of existing K12 artifacts or act as a decision-making member multi-agent ensemble.

\section*{Acknowledgments}
We thank Ben Domingue and Chris Piech for their time and insights. We are also grateful for the feedback from members of Stanford Trustworthy AI Research (STAIR) Lab, Stanford Lytics Lab,  Piech Lab, Stanford's Social NLP working group, and Stanford's ``Researching, Presenting and Publishing Work in AI \& Education'' Group from Spring 2025. Finally, we thank the reviewers for their constructive feedback.


\bibliography{references}

\appendix

\section{Datasets for Supporting Classrooms}
Meaningful representation of student classroom learning is absent on the internet largely for laws protecting children’s privacy. The text for the data was only made anonymized and public in 2023 \citep{demszky_ncte_2023} and is not ``crawlable'' nor is it easy to link these data with the (also not crawlable) annotations and VAMs \cite{kane_national_2015}. It is the linking of these two sources that makes it 1 of 2 extant datasets having classroom interactions with both expert ratings on real teaching instruments and VAMs. The second dataset (the “MET Project”) is not publicly available. Part of our interest in doing this research is because conducting these studies is extremely expensive. (At least 2 organizations are trying to collect more data that meet these criteria.) Regardless, these datasets are notoriously challenging; \citeauthor{xu_promises_2024} use ½ page to articulate challenges with this data. \citeauthor{ho_reliability_2013} give a measure of success for humans at reliability of 0.65--our initial target for using LLMs. LLMs initially demonstrated high consistency (increasing reliability), but then we discovered the misalignment. Ideally success would lead to ratings that improve human reliability \cite{hardy_all_2025} and alignment with robust checks to avoid unhelpful LLM contributions.

\subsection{Observation Instruments}\label{apx:mqi_class}

For each of the observation instruments, the abbreviation codes used in this study are listed with the expanded names in Table \ref{table:items}. The distributions of scores across all items for all rater families are in Figure \ref{fig:ratingdistributions}. The CLASS rubric has 12 items on a scale from 1 to 7, rated at 15 minute intervals. The MQI rubric has 13 items on a scale from 1 to 3, rated at 7.5 minute intervals.

The mathematics-specific MQI and general CLASS frameworks are explicitly and implicitly capturing different dimensions of classroom instruction \cite{blazar_attending_2017} for human raters. 

\begin{figure*}[htbp]
    \centering
    \includegraphics[width=1\linewidth]{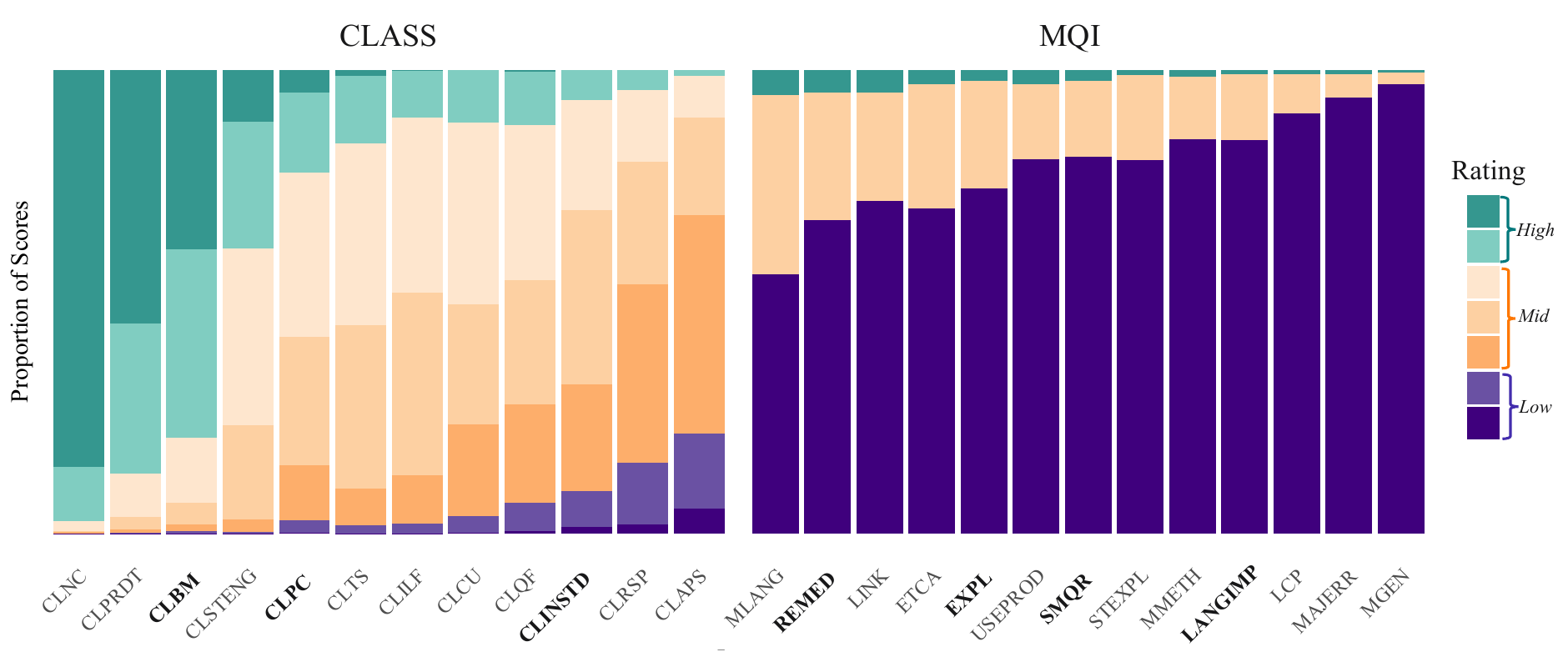}
    
    \caption{Proportions of human rater scores by item.}
    \label{fig:ratingdistributions}
\end{figure*}

    

\paragraph{The MQI Framework (13 items)}
The MQI framework is intended to detect specific classroom practices, assuming that much of the classroom discourse will not be relevant for each item and are one-inflated. The 13 MQI items\footnote{Instruments for classroom instruction are composed of multiple items that represent distinct instructional dimensions to be evaluated} within the dataset have at least two human raters per classroom observation. The MQI instrument has four instructional domains that capture information on the quality of teaching and learning: Richness of Mathematics, Working with Students, Student Participation, and Errors and Imprecisions. This paper will focus on one item from each of the four MQI domains: teacher explanations (EXPL), remediation of student errors (REMED), student questioning and reasoning (SMQR), and imprecision in mathematical language (LANGIMP), respectively. For ease of interpretation in this study, LANGIMP is reverse-coded, so higher scores are better. (For additional information about the MQI instrument, see \cite{hill_mathematical_2008, hardy_all_2024,hill_when_2012,kane_gathering_2012}).

\paragraph{The CLASS Framework (12 items)}
The CLASS items capture broader constructs and assume that most of the classroom discourse is relevant. These differences can be seen visually in Figure \ref{fig:ratingdistributions}. The 12 CLASS items have one raters per classroom observation. Prior work has shown that the CLASS items generally have higher human agreement than MQI items \cite{kane_gathering_2012,kane_national_2015}. Three items from the CLASS instrument will be analyzed, following prior work \cite{wang_is_2023}: behavior management (CLBM), instructional dialogue (CLINSTD), and positive classroom climate (CLPC).

\subsection{Replication Study Details}
The present study contains two replication studies; we used the same settings, sampling, and data as specified in each original \citep{wang_is_2023,hardy_all_2025}. We did not alter the prompts, samples, nor settings when replicating the original studies. Footnote 3 has a link to the full replication test set and access to the full prompts used. Wang \& Demsky (2023) also provide a link on the first page of their study (as well as many appendix pages of prompt templates). We used the same hyperparameter settings (temperature 0, completion) as found in the original code. 

Additionally, for verifiability and replication, Footnote 3 also has the outputs of each prompt in the present study (where failure modes can be identified and claims verified), with lookup columns that match the data from Wang \& Demszky. 
\subsubsection{Test Set and Tasks}
Building on prior work, we adopt the zero-shot test set and prompts established by \cite{wang_is_2023}. This test set consists of a stratified sample of lessons from the NCTE data. Our evaluation covers seven distinct rating tasks (i.e., instrument items): four from the MQI framework (teacher explanations, remediation of student errors, student questioning, and precision of language) and three from the CLASS framework (behavior management, instructional dialogue, and positive climate). For the measures of alignment, we maximize  content differences by selecting one item from each of the four empirical factors identified in \cite{blazar_attending_2017,kane_national_2015}. We select the items with the highest interrater reliability (Cohen's $\kappa$ in Appendix of \cite{kane_national_2015}) where available, otherwise, test set item with the largest factor loading. 

\subsubsection{Models and Prompts}
We selected 4 of the LLMs that performed the best at the time of the experiment, chosen from the HELM leaderboard \cite{liang_holistic_2023} and the ``Pedagogy Benchmark'' \cite{lelievre_benchmarking_2025}, and then prioritized diversity of models. We also included Google's \textit{LearnLM} due to its stated focus on education \cite{learnlm_learnlm_2024}. Following prior work \cite{liang_holistic_2023,wang_is_2023}, we query each model using three distinct prompting strategies for each task: (1) a \textbf{base prompt} with the core instructions, (2) a \textbf{chain-of-thought} prompt encouraging step-by-step reasoning \cite{wei_chain--thought_2023}, and (3) a prompt that acts like a \textbf{retrieval-augmented generation (RAG)} by including additional, relevant task-specific rubric information. All for LLM calls are available online.
\footnote{\url{https://github.com/hardy-education/measuring_llm_impact_alignment}}  
All model predictions are publicly available to support future research for reproducibility (see Footnote 3). 

Below are three examples for three different tasks for each of the three prompt template families from \citeauthor{wang_is_2023} (where all examples of other prompt templates can be found): Base prompt for the Remediation of Student Errors (REMED) task, a RAG-like, additional-details prompt template for the Language Imprecision (LANGIMP) task, and the reasoning/chain-of-thought template for the Classroom instructional dialogue (CLINSTD) task.

REMED (Base):

\texttt{\\
Consider the following classroom transcript.  
\\
\\
Transcript: \{transcript\}
\\
\\
Based on the classroom transcript, rate the teacher’s degree of remediation of student errors and difficulties on a scale of 1-3 (low-high). This means that the teacher gets at the root of student misunderstanding, rather than repairing just the procedure or fact. This is more than a simple correction of a student mistake.
\\
\\
Rating (only specify a number between 1-3):}

LANGIMP (simple RAG-like/extra information)

\texttt{\\
Consider the following classroom transcript.  
\\
\\
Transcript: \{transcript\}
\\
\\
Based on the classroom transcript, rate the teacher’s imprecision in language or notation on a scale of 1-3 (low-high). The teacher’s imprecision in language or notation refers to problematic uses of mathematical language or notation. For example, errors in notation (eg. mathematical symbols), in mathematical language (eg. technical mathematical terms like ``equation'') or general language (eg. explaining mathematical ideas or procedures in non-technical terms). Do not count errors that are noticed and corrected within the segment.  
\\ \\
Explanation of ratings: \\
1: Brief instance of imprecision. Does not obscure the mathematics of the segment. \\
2: Imprecision occurs in part(s) of the segment or imprecision obscures the mathematics but for only part of the segment. \\
3: Imprecision occurs in most or all of the segment or imprecision obscures the mathematics of the segment.
\\
\\
Rating (only specify a number between 1-3):}

CLINSTD (Reasoning)

\texttt{\\
Consider the following classroom transcript.  
\\
\\
Transcript: \{transcript\}
\\
\\
Please do the following. \\ 
1. Think step-by-step how you would rate the instructional dialogue of the teacher on a scale of 1-7 (low-high). Instructional dialogue captures the purposeful use of content-focused discussion among teachers and students that is cumulative, with the teacher supporting students to chain ideas together in ways that lead to deeper understanding of content. Students take an active role in these dialogues and both the teacher and students use strategies that facilitate extended dialogue. \\
2. Provide your rating as a number between 1 and 7.  Format your answer as: Reasoning: Rating (only specify a number between 1-7):  \\ \\ 
Reasoning:
}

Additionally, while not the focus of this study, we replicated the SOTA models of \cite{hardy_all_2025} to have confidence that the misalignment we observed were not the result of an impossible task using only transcripts. These encoders and those from \cite{hardy_all_2025} are shown as baselines in Fig. \ref{fig:alignment}.  This results in 103,148 total observations across models, tasks and prompts.  Additionally, while not the focus of this study, we replicated the SOTA models of \cite{hardy_all_2025} to have confidence that the misalignment we observed were not the result of an impossible task using only transcripts. These encoders and those from \cite{hardy_all_2025} are shown as baselines in Fig. \ref{fig:alignment}.

\begin{figure*}[htbp]
    \centering
    \includegraphics[width=1\linewidth]{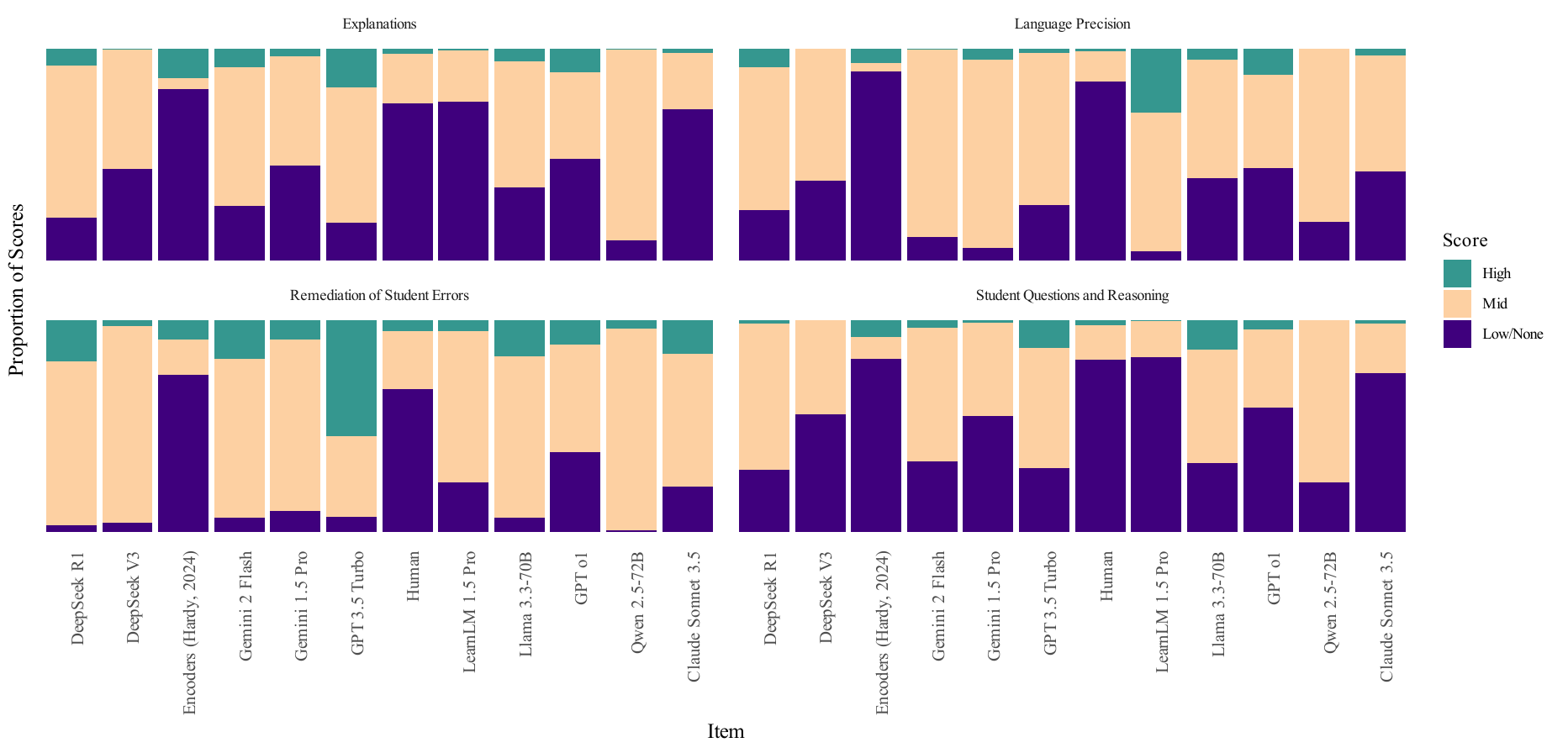}
    
    \caption{Proportions of rater scores by MQI item.}
    \label{fig:ratingdistributions_mqi}
\end{figure*}

    

\begin{table*}[h]
    \centering
    \setlength\tabcolsep{2pt}\renewcommand\defaultaddspace{1.5ex}
        \caption{CLASS and MQI item descriptions and corresponding abbreviations from the test set of \citeauthor{wang_is_2023}. \dag denotes items that are reverse coded due to being negatively worded with respect to the construct of teacher ability. \textbf{Bolded} items are the focus items of the present alignment study. Bracketed item descriptions are the names of the factors with highest identified loadings by \citeauthor{blazar_attending_2017} per category in Appendix 2.b of the original study \cite{kane_national_2015}. While in the present impact alignment study we only evaluated four items, we performed a full replication of all seven items, the responses and outputs of which can be found in the online data. Variance decompositions for all seven items on the test set can be found in Appendix \ref{apx:vardecomp}.}
    \begin{tabularx}{1\textwidth}{@{} l>{\hsize=0.65\hsize}X>{\hsize=1.30\hsize}X>{\hsize=1.05\hsize\arraybackslash} X @{}}
    \toprule
    \textbf{Item} & \textbf{Item Name} & \textbf{[Factor] Item Description} \\
    \hline
    \hline
    \underline{\textbf{MQI}} & \\
    EXPL &\textit{Teacher Explanations} & [Mathematical Instruction] Teacher explanations that give meaning to ideas, procedures, steps, or solution methods. \\
    \textbf{LANGIMP}\dag & \textit{Imprecision in Language or Notation} & [Mathematical Errors] Imprecision in language or notation, with regard to mathematical symbols and technical or general mathematical language.  \\
    \textbf{REMED} & \textit{Remediation of Student Errors and Difficulties} & [Mathematical Instruction] Remediation of student errors and difficulties addressed in a substantive manner. \\
    SMQR & \textit{Student Mathematical Questioning and Reasoning} & [Mathematical Instruction] Student mathematical questioning and reasoning, such as posing mathematically motivated questions, offering mathematical claims or counterclaims. \\
    \midrule
    \underline{\textbf{CLASS}}  & \\
    CLPC & \textit{Classroom Positive Climate} & [General Instruction] The relationships among teachers and students, and the warmth, respect, and enjoyment communicated by verbal and nonverbal interactions.\\
    \textbf{CLBM} & \textit{Behavior Management} &  [Classroom Organization] The use of effective methods to encourage desirable behavior and prevent and redirect misbehavior. \\
    \textbf{CLINSTD} & \textit{Instructional Dialogue} & [General Instruction] The purposeful use of dialogue across the class to facilitate students’ understanding of content including structured, cumulative questioning and discussion which guide and prompt students. \\
    \bottomrule
    \end{tabularx}

    \label{table:items}
\end{table*}

\section{Challenges in VAM signal}
Human-VAM alignment for a given 15 minute rating on one dimension of teaching will necessarily be small. But across thousands of observations, we would expect better teaching to be aligned with better outcomes.  To test this before the main study, we performed 6 robustness checks, excluded only for brevity, including highly robust estimators between expert ratings and VAM. Summary below:
\subsection{Thiel-Sen Estimator}
Thiel-Sen (TS) estimator \citep{sen_estimates_1968} is a non-parametric method very closely related to Kendall's $\tau$. It computes the median of all slopes between point pairs: $\hat{\beta} = \text{median}\left(\frac{y_j - y_i}{x_j - x_i}\right)$. We would expect a positive median slope, $\hat{\beta} < 0$ as a direct analogue of the expected relationship of Kendall's $\tau$ \citep{kendall_new_1938}. 
\subsection{Repeated Median Estimator}
Siegel’s Repeated Median Estimator (RME) is a second highly robust estimator of linear relationships with the highest breakdown point of 50\% \citep{siegel_robust_1982}. It estimates the slope of the regression line $y = A + Bx$ for a set of points. For our purposes, if humans raters fail to have a positive RME slope with respect to VAM on an item, the item either has too high a proportion of bad data or no meaningful relationship. $\hat{B} = \underset{i}{\operatorname{median}} \ \underset{j\,\ne\,i}{\operatorname{median}} \ (Y_j - Y_i) / (X_j - X_i)$ where the inner median is the median of the slopes connected to this observation, similar to the Thiel-Sen Estimator, and the outer median is the median slope across all items for the inner slopes. For signal to be robustly identified, we expert humans ratings should have a positive relationship with VAM, measured by $\tau$, the 90\% lower CI of $\tau$, TS, and RME. We also test that $\tau$ is greater than the $\tau$ generated by random sampling and we perform a quartile test for difference between rating Q1 and Q4 and VAM \citep{kane_estimating_2008}, as shown in Table \ref{tab:robustness}.

\begin{table*}[!htbp]
\centering
\small
\caption{Robustness Test Results Across Different Tasks}
\label{tab:robustness}
\begin{tabular}{lcccccccc}
\toprule
Task & $\tau > 0$ & \shortstack{Low. 90\% \\ CI $> 0$} & \shortstack{Thiel Sen \\ $> 0$} & RME $> 0$ & \shortstack{$T >$ \\ Rand} & $T \geq T$. Exp & $Q_4>Q_1$ & \shortstack{Robust. \\ Pass Rate} \\
\midrule
CLBM    & 0.07, Y & 0.06***, Y & 0.013, Y & 0.008, Y & $-0.04$, Y & \shortstack{0.0004, Y \\ ($-0.04$, 0.04)} & 0.02**, Y & 100 (7/7) \\
CLINSTD & 0.02, Y & 0.003*, Y  & 0.003, Y & 0.0002, Y & 0.004, Y & \shortstack{0.0004, Y \\ ($-0.04$, 0.04)} & 0.02**, Y & 100 (7/7) \\
LANGIMP & 0.12, Y & 0.10***, Y & 0.06, Y  & 0.05, Y   & 0.03, Y  & \shortstack{0.04, Y \\ (0.000, 0.08)} & 0.03**, Y & 100 (7/7) \\
REMED   & 0.02, Y & 0.001*, Y  & 0.008, Y & 0.004, Y  & $-0.01$, Y & \shortstack{0.04, Y \\ (0.000, 0.08)} & 0.01, N   & 71 (5/7) \\
\bottomrule
\end{tabular}
\begin{tablenotes}
\footnotesize
\item Note: Y indicates the test passed, N indicates failure. Significance levels: * $p < 0.10$, ** $p < 0.05$, *** $p < 0.01$.
\end{tablenotes}
\end{table*}

Given the tests in Table \ref{tab:robustness}, we take the position that, while the signal is faint, it is indeed positive and distinguishable from the noise. One way to contextualize how we can extract the signal over many observations is to consider that VAMs have about the same accuracy as predicting year-over-year ERA\footnote{Earned Run Average (ERA) is the average number of earned runs a pitcher allows per nine innings, calculated as (Earned Runs \(\times \) 9) / Innings Pitched. It is the primary statistic for measuring pitcher effectiveness.} of professional baseball pitchers \citep{mccaffrey_intertemporal_2009}. This task is harder, as would be the predicting of pitcher ERA from only watching a few innings of one game: on its own, there isn't enough signal, but in aggregate we assume that there is.

\subsection{VAMs in other Research}\label{app:priorvam}
Prior rater-to-VAM studies that used a much stronger signal by aggregating human rater scores and employing the far more forgiving Pearson correlation found that the significant overall relationships for MQI items were 0.09 and for CLASS items were 0.18 \citep{kane_gathering_2012}, with statistical significance despite the noise, compared to the Kendall $\tau_b$ 0.11/0.03 and 0.14/0.06 reported for the humans in Figure \ref{fig:alignment}. Concerned with the question of text-only signal identification, we also did the replication of \cite{hardy_all_2024} seen as the green diamonds in Figure \ref{fig:alignment}.

\subsection{High-noise Context Measurement}\label{sec:highnoise}
\subsubsection{Failure Mode of LLM Behavioral Homogeneity}
Behavioral homogeneity, in various forms, has been in literature and should not be surprising given what we know about pretraining \citep{mccoy_embers_2023}. Behavioral homogeneity is the first of the five studies, and should not be seen as the core finding: similar to contemporary studies on LLM correlated errors \citep{kim_correlated_2025}, we first demonstrate their existence. A key idea of the alignment studies herein is highlighting the importance of scrutinizing our evaluations and working on problems that are not easily verifiable. There need to be methods for measuring, even what may be already known, in contexts where measurement has previously been too difficult.

We posit that the knowledge of the types of failure modes found within this study have been in the literature for a long time, but that it may not be a shared belief across NLP researchers as LLMs increase in size and in ultracrepidarianism. K12 education is not the only domain affected by rushing research because the tasks of interest are very difficult to verify. Efforts to understand the economic impact\footnote{We note that this study evaluates downstream task, but does not, in fact, evaluate the ultimate relationship of the economic impact of the quality of the downstream task outputs.} of LLMs on professionals, such as GDPval \citep{patwardhan_gdpval_2025}, have systematically excluded classrooms and contexts involving children. Education is not alone in this either; social work, therapy, and other sectors are also affected beyond the ``top 9''. The issues of being able to produce meaningful measurements where F1 scores above 0.95 affect much more than education.

Thus, an interesting contribution of this paper is methodological: robust and principled methods for measuring outcomes in high-noise contexts. In our opinion, it is less interesting that our first finding is a possibly predictable failure mode and more interesting that our methods allowed us to detect it, given the high noise contexts. This may provide hope for other domains where downstream tasks are difficult to measure and intended impacts even more so. Detecting these failure modes, and then measuring their augmentation as LLMs are used in ensemble, is novel in these contexts.

\section{Variance Decomposition Model}\label{apx:vardecomp}

This appendix formalizes the structural decomposition used to attribute observed misalignment error to facets of a fully crossed evaluation design, introduced in Section~\ref{sec:vardecomp} and discussed in Section~\ref{sec:discvardec}. The aim is not merely to report error, but to identify which portions of error are (i) developer-controllable (\textsc{LLM}, \textsc{Prompt}), (ii) systemic and shared across implementations (\textsc{Item}, \textsc{Obs}, and their interactions), and (iii) irreducibly idiosyncratic at the granularity of a single score (cell-specific residual variation). While our application concerns classroom instruction, the design and estimands apply broadly to LLM evaluation whenever outputs are compared to noisy, multifaceted targets.

We present the framework with sufficient generality that it can be applied to any fully crossed evaluation design, not only in education but wherever LLM outputs are compared against noisy, multi-faceted ground truth.  We are not aware of prior use of this method in the ACL literature, and we hope it proves useful to researchers seeking to disentangle controllable from irreducible sources of error in high-noise evaluation contexts.
We hope the methods will support others in being able to understand how much of the error they are seeing in their models is irreducible in contexts where their signal is clouded by many facets of variation \citep{brennan_generalizability_2001,brennan_advanced_2001,brennan_coecients_2003}. 
Education happened to be the content matter of this study, but the study is about the ability to measure in high-noise contexts, with noisy labels and mismatched grain sizes. We hope this work helps to create a middle ground between “gold standard” labels and qualitative research: the “pewter standard” labels, perhaps.


\subsection{Model specification}\label{apx:modelspec}

Let the index set of the fully crossed design be $\mathcal{D} = \mathcal{C} \times \mathcal{I} \times \mathcal{M} \times \mathcal{P}$, where $\mathcal{C}$ denotes observed classroom transcript segments ($|\mathcal{C}| = N_C$), $\mathcal{I}$ rubric items ($|\mathcal{I}| = N_I$), $\mathcal{M}$ foundation models ($|\mathcal{M}| = N_M = 16$), and $\mathcal{P}$ prompting strategies ($|\mathcal{P}| = N_P = 3$).  For each $(c,i,m,p) \in \mathcal{D}$, let $\widetilde{X}_{cimp}$ be the standardized\footnote{The findings are robust to standardization types. We standardize here by dividing by two standard deviations (instead of one) to better align the relative scales.} LLM rating and $\widetilde{Y}_c$ the (pre-standardized) stacked value-added measure for classroom $c$.  The misalignment error is
\begin{equation}\label{eq:sqerr}
  \hat{e}_{cimp} \;=\; \bigl(\widetilde{X}_{cimp} - \widetilde{Y}_c\bigr)^2.
\end{equation}

\paragraph{Random-effects decomposition.}
We fit a fully crossed random-effects model over all main effects and interactions:
\begin{equation}\label{eq:apx_fullmodel_rewrite}
\hat{e}_{cimp}
= \mu \;+\; \sum_{\emptyset\neq \alpha \subseteq \{c,i,m,p\},\,|\alpha|\le 3}\nu_{\alpha} \;+\; \eta_{cimp},
\end{equation}
where each $\nu_\alpha$ is a mean-zero random effect indexed by the corresponding facet(s), $\nu_{\alpha}\sim\mathcal{N}(0,\sigma^2_{\alpha})$, mutually independent across $\alpha$, and $\eta_{cimp}\sim\mathcal{N}(0,\sigma^2_{\eta})$ is the cell-specific remainder. Because each $(c,i,m,p)$ cell is observed once, the four-way interaction is not separately identifiable from the residual; hence we report the combined term as
\[
\sigma^2_{cimp+\epsilon}\;:=\;\sigma^2_{\textsc{LLM:Prompt:Item:Obs}}+\sigma^2_{\epsilon}.
\]

\paragraph{Expanded Random-effects.} More explicitly, we can reformulate Eq. \ref{eq:apx_fullmodel_rewrite} in expanded form the disaggregated sum of random effects corresponding to all main effects and interactions in the fully crossed $I \times M \times P \times C$ design:
\begin{align}\label{eq:vardecomp_full}
  \hat{e}_{cimp} \;=\;\; & \mu + \underbrace{\nu_c + \nu_i + \nu_m + \nu_p}_{\text{main effects}} \nonumber\\
  &+ \underbrace{\nu_{ci} + \nu_{cm} + \nu_{cp} + \nu_{im} + \nu_{ip} + \nu_{mp}}_{\text{two-way interactions}} \nonumber\\
  &+ \underbrace{\nu_{cim} + \nu_{cip} + \nu_{cmp} + \nu_{imp}}_{\text{three-way interactions}} \nonumber\\
  &+ \;\varepsilon_{cimp},
\end{align}
where $\mu$ is the grand mean, each random effect is normally distributed with mean zero,
$  \nu_{\alpha} \sim \mathcal{N}(0,\,\sigma^2_{\alpha}) \quad \text{for each index set } \alpha \subseteq \{c,i,m,p\},\; |\alpha|\ge 1$,
and $\varepsilon_{cimp} \sim \mathcal{N}(0,\sigma^2_{\varepsilon})$ is the residual.  All random effects are mutually independent.  

\paragraph{Variance share}
Let $\mathcal{S}$ denote the set of modeled variance components (all $\sigma^2_{\alpha}$ plus $\sigma^2_{\eta}$). The \emph{variance share} of component $\alpha$ is
\begin{equation}\label{eq:varshare}
  \pi_{\alpha} \;=\; \frac{\sigma^2_{\alpha}}{\sigma^2_{\mathrm{total}}}, \; \sigma^2_{\mathrm{total}} = \sum_{\alpha}\sigma^2_{\alpha} + \sigma^2_{\varepsilon}=\sum_{k\in\mathcal{S}}\sigma^2_k.
\end{equation}
\noindent Posterior estimates of $\pi_\alpha$ for all 15 components appear in Table~\ref{tab:sq_err_decomp} (summary) and Table~\ref{tab:app_sq_decomp_full} (full fit statistics including mean, SD, frequentist-style CI, and $\hat{R}$).


\subsection{Sign-preserving estimations for shared behaviors in errors}\label{apx:signpreserve}

The squared-error formulation in Eq.~\ref{eq:sqerr} and reported in Table \ref{tab:app_sq_decomp_full} discards the direction of misalignment (whether an LLM over- or under-rates relative to VAM).  To recover this information and better estimate the degree to which LLM errors move together, we fit two additional models using the \emph{signed magnitude}:
\begin{equation}\label{eq:signerr}
  \hat{e}^{\pm}_{cimp} \;=\; \bigl|\widetilde{X}_{cimp} - \widetilde{Y}_c\bigr|\;\cdot\;\bigl(\widetilde{X}_{cimp} - \widetilde{Y}_c\bigr),
\end{equation}
which preserves the sign while retaining quadratic scaling. We estimate it twice: once using the items from the present study $\hat{e}^{\pm}$ and a second estimation, $\hat{e}^{\pm}_{(+)}$, includes all of the items from the original replication study (for a total of seven) to provide more observations of behavior.  The results of these decompositions are in Tables \ref{tab:app_sgn_sq_decomp_full} and \ref{tab:app_sgn_sq_decomp_full_addl_items}, respectively. 

From this latter estimation we find that at most $56.4\%\pm10.2\%, \hat{R}=0.99$ of the variance in signed error is attributable to factors involving the LLM or prompt, reinforcing the conclusion that roughly half of the misalignment variance arises from facets outside a developer's control. Note that these additional three items did not undergo robustness checks and were not used in the main body of the study because they had weaker loadings on constructs than other items in their same factor category \citep{blazar_attending_2017}. See the parallel analyses performed on the human data in Table \ref{tab:hum_err_decomp}.

\subsection{Bayesian estimation details}\label{apx:estimation_rewrite}

We estimate Eq.~\ref{eq:apx_fullmodel_rewrite} using Bayesian multilevel modeling (\texttt{brms}). Each standard deviation parameter receives a weakly informative half-$t$ prior:
\[
\sigma_k \sim \mathrm{Student\text{-}t}^{+}(3,0,2.5)\qquad \forall k\in\mathcal{S}.
\]
Convergence is assessed via rank-normalized split-$\hat{R}$; all reported models satisfy $\hat{R}\approx 1$ (see tables). A frequentist REML fit with the identical formula yields consistent component ordering and comparable point estimates, providing a second check that the variance partition is not an artifact of Bayesian regularization. $\hat{R}$ values were computed using rank-normalized split chains following \citet{vehtari_rank-normalization_2021}.  The formula and all hyperparameters (iterations, chains, cores, thinning) are shown in Figure~\ref{fig:bayes}.

\begin{figure*}
    \centering
    \includegraphics[width=1\linewidth]{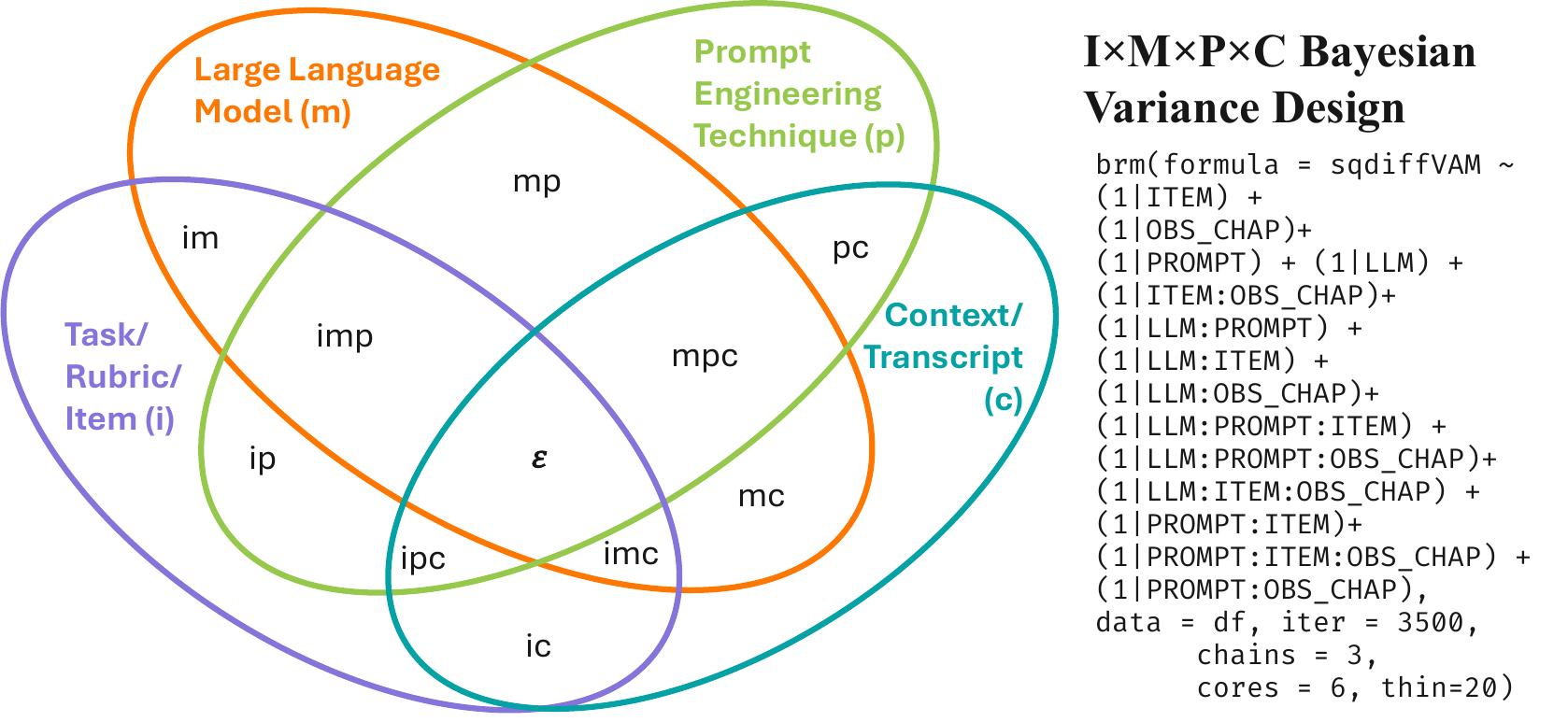}
    \caption{\textbf{Bayesian Error Variance Decomposition}: (\textit{left}) fully crossed facet diagram for sources of variance. (\textit{right}) corresponding \texttt{brms} code listing. 
    }
    \label{fig:bayes}
\end{figure*}

\begin{table*}[ht]
\centering
\caption{\textbf{Bayesian Squared Error Variance Decomposition} $(\hat{e})$ Parameter Estimates and Fit Statistics}\label{tab:app_sq_decomp_full}
\resizebox{\ifdim\width>\linewidth\linewidth\else\width\fi}{!}{
\begin{tabular}{lcccccccc}
\toprule
variable & median & MAD & MAP & HDI & mean & SD & CI & $\hat{R}$\\
\midrule
ITEM & 0.08 & 0.08 & 0.03 & [0.01,0.62] & 0.14 & 0.17 & [0.01,0.54] & 0.99\\
LLM & 0.07 & 0.04 & 0.06 & [0.02,0.18] & 0.08 & 0.05 & [0.02,0.16] & 1.00\\
OBS & 0.02 & 0.01 & 0.02 & [0,0.03] & 0.02 & 0.01 & [0.01,0.03] & 1.01\\
PROMPT & 0.02 & 0.03 & 0.00 & [0,0.58] & 0.08 & 0.15 & [0,0.46] & 1.00\\
\addlinespace

ITEM:OBS & 0.04 & 0.01 & 0.05 & [0.01,0.06] & 0.04 & 0.01 & [0.02,0.06] & 1.00\\
LLM:ITEM & 0.03 & 0.01 & 0.03 & [0.01,0.06] & 0.03 & 0.01 & [0.01,0.05] & 1.00\\
LLM:OBS & 0.00 & 0.00 & 0.00 & [0,0] & 0.00 & 0.00 & [0,0] & 1.00\\
LLM:PROMPT & 0.02 & 0.01 & 0.01 & [0,0.04] & 0.02 & 0.01 & [0.01,0.03] & 1.00\\
PROMPT:ITEM & 0.01 & 0.01 & 0.00 & [0,0.05] & 0.01 & 0.01 & [0,0.03] & 1.01\\
PROMPT:OBS & 0.00 & 0.00 & 0.00 & [0,0] & 0.00 & 0.00 & [0,0] & 1.00\\
\addlinespace
LLM:ITEM:OBS & 0.19 & 0.03 & 0.21 & [0.05,0.23] & 0.18 & 0.05 & [0.07,0.23] & 0.99\\
LLM:PROMPT:ITEM & 0.03 & 0.01 & 0.03 & [0.01,0.05] & 0.03 & 0.01 & [0.01,0.05] & 1.01\\
LLM:PROMPT:OBS & 0.14 & 0.02 & 0.15 & [0.04,0.17] & 0.13 & 0.03 & [0.05,0.17] & 1.00\\
PROMPT:ITEM:OBS & 0.02 & 0.00 & 0.02 & [0.01,0.03] & 0.02 & 0.01 & [0.01,0.03] & 1.00\\
LLM:PROMPT:ITEM:OBS +  $\epsilon$ & 0.24 & 0.04 & 0.26 & [0.06,0.29] & 0.22 & 0.06 & [0.09,0.28] & 1.00\\
\bottomrule
\end{tabular}}
\end{table*}

\begin{table*}[h]
\centering
\caption{\textbf{Bayesian Signed Squared Error Variance Decomposition} $(\hat{e}^\pm)$ Parameter Estimates and Fit Statistics}\label{tab:app_sgn_sq_decomp_full}
\resizebox{\ifdim\width>\linewidth\linewidth\else\width\fi}{!}{
\begin{tabular}{lcccccccc}
\toprule
variable & median & MAD & MAP & HDI & mean & SD & CI & $\hat{R}$\\
\midrule
ITEM & 0.03 & 0.04 & 0.01 & [0,0.34] & 0.07 & 0.09 & [0,0.22] & 1.00\\
LLM & 0.09 & 0.04 & 0.07 & [0.03,0.24] & 0.10 & 0.05 & [0.03,0.19] & 1.01\\
OBS & 0.03 & 0.01 & 0.03 & [0.01,0.05] & 0.03 & 0.01 & [0.01,0.04] & 1.01\\
PROMPT & 0.00 & 0.01 & 0.00 & [0,0.5] & 0.06 & 0.13 & [0,0.47] & 1.05\\
\addlinespace

ITEM:OBS & 0.12 & 0.02 & 0.12 & [0.06,0.15] & 0.11 & 0.02 & [0.07,0.14] & 1.05\\
LLM:ITEM & 0.06 & 0.02 & 0.06 & [0.03,0.1] & 0.06 & 0.02 & [0.03,0.1] & 1.02\\
LLM:OBS & 0.00 & 0.00 & 0.00 & [0,0] & 0.00 & 0.00 & [0,0] & 1.02\\
LLM:PROMPT & 0.01 & 0.01 & 0.01 & [0,0.03] & 0.01 & 0.01 & [0,0.03] & 1.01\\
PROMPT:ITEM & 0.00 & 0.00 & 0.00 & [0,0.03] & 0.00 & 0.01 & [0,0.02] & 1.02\\
PROMPT:OBS & 0.00 & 0.00 & 0.00 & [0,0] & 0.00 & 0.00 & [0,0] & 1.00\\
\addlinespace

LLM:ITEM:OBS & 0.20 & 0.02 & 0.21 & [0.1,0.24] & 0.19 & 0.03 & [0.11,0.23] & 1.06\\
LLM:PROMPT:ITEM & 0.06 & 0.01 & 0.06 & [0.03,0.08] & 0.06 & 0.01 & [0.03,0.08] & 1.06\\
LLM:PROMPT:OBS & 0.13 & 0.02 & 0.14 & [0.07,0.16] & 0.13 & 0.02 & [0.07,0.15] & 1.05\\
PROMPT:ITEM:OBS & 0.02 & 0.00 & 0.02 & [0.01,0.03] & 0.02 & 0.00 & [0.01,0.03] & 1.06\\
LLM:PROMPT:ITEM:OBS +  $\epsilon$ & 0.17 & 0.02 & 0.17 & [0.08,0.19] & 0.16 & 0.03 & [0.09,0.19] & 1.05\\
\bottomrule
\end{tabular}}
\end{table*}

\begin{table*}[h]
\centering
\caption{\textbf{Bayesian Signed Squared Error Extended Variance Decomposition $(\hat{e}^\pm_{(+)})$} Parameter Estimates and Fit Statistics for all replication study items. Note the large difference in Item variation when other items are included and the decrease in Prompting volatility.}\label{tab:app_sgn_sq_decomp_full_addl_items}
\resizebox{\ifdim\width>\linewidth\linewidth\else\width\fi}{!}{
\begin{tabular}{lcccccccc}
\toprule
variable & median & MAD & MAP & HDI & mean & SD & CI & $\hat{R}$\\
\midrule
ITEM & 0.29 & 0.13 & 0.27 & [0.14,0.63] & 0.33 & 0.14 & [0.15,0.59] & 0.99\\
LLM & 0.04 & 0.03 & 0.04 & [0.01,0.16] & 0.05 & 0.04 & [0.01,0.12] & 0.99\\
OBS & 0.07 & 0.02 & 0.07 & [0.03,0.09] & 0.07 & 0.02 & [0.04,0.09] & 1.00\\
PROMPT & 0.00 & 0.00 & 0.00 & [0,0.23] & 0.02 & 0.06 & [0,0.13] & 1.01\\
\addlinespace
ITEM:OBS & 0.09 & 0.02 & 0.09 & [0.04,0.12] & 0.08 & 0.02 & [0.05,0.11] & 1.01\\
LLM:ITEM & 0.04 & 0.01 & 0.04 & [0.02,0.07] & 0.04 & 0.02 & [0.02,0.07] & 1.00\\
LLM:OBS & 0.01 & 0.00 & 0.01 & [0,0.01] & 0.01 & 0.00 & [0.01,0.01] & 1.00\\
LLM:PROMPT & 0.01 & 0.01 & 0.01 & [0,0.03] & 0.01 & 0.01 & [0,0.03] & 1.01\\
PROMPT:ITEM & 0.00 & 0.00 & 0.00 & [0,0.01] & 0.00 & 0.00 & [0,0.01] & 1.00\\
PROMPT:OBS & 0.00 & 0.00 & 0.00 & [0,0] & 0.00 & 0.00 & [0,0] & 1.02\\
\addlinespace
LLM:ITEM:OBS & 0.12 & 0.02 & 0.13 & [0.06,0.16] & 0.12 & 0.03 & [0.07,0.16] & 0.99\\
LLM:PROMPT:ITEM & 0.05 & 0.01 & 0.05 & [0.02,0.07] & 0.05 & 0.01 & [0.03,0.07] & 1.00\\
LLM:PROMPT:OBS & 0.04 & 0.01 & 0.05 & [0.02,0.06] & 0.04 & 0.01 & [0.03,0.05] & 0.99\\
PROMPT:ITEM:OBS & 0.03 & 0.01 & 0.03 & [0.01,0.03] & 0.02 & 0.01 & [0.01,0.03] & 0.99\\
LLM:PROMPT:ITEM:OBS +  $\epsilon$  & 0.15 & 0.03 & 0.16 & [0.07,0.2] & 0.15 & 0.03 & [0.09,0.19] & 0.99\\
\bottomrule
\end{tabular}}
\end{table*}


\subsection{Generalizability and decision studies}\label{app:gdstudy}

Having estimated variance components, we use the Generalizability Theory decision-study (D-study) framework~\citep{brennan_multifacet_2001-1} to ask a forward-looking question: \emph{how many LLMs and prompting strategies would one need to average over in order to obtain a stable estimate of the shared, persistent error signal?}

\subsubsection{Generalizability Design}

In a generalizability framework, the three prompt families (base, chain-of-thought, pseudo-RAG) are treated as sampled operationalizations from a broader universe of commonly used strategies.  The key claim is not that these specific prompts failed, but that variance attributable to the prompt facet---including its interactions with LLM---is small relative to other sources.  This rests on the full variance decomposition, not on the prompt main effect alone.

We define the \emph{object of measurement} as the item--observation combination $(i,c)$, and the \emph{instrumentation facets} as LLM ($m$) and prompt ($p$).  Averaging over $n_m$ models and $n_p$ prompts, the relative and absolute reliability coefficients are given by the following.

\begin{definition}[Relative reliability]\label{prop:erho}
Let $S$ denote the set of all variance components estimated in Eq.~\ref{eq:vardecomp_full}, and let $S_{\alpha} = \{\alpha \in S
\}$ be the subset containing the object of measurement, $\alpha$.  The expected relative reliability for $n_m$ models and $n_p$ prompts is
\begin{equation}\label{eq:dstudyErho}
  E\widehat{\rho}_\alpha^2 \;=\; \frac{\sigma^2_{\alpha}}{\displaystyle\sum_{k \,\in\, S_{\alpha}} \frac{\sigma^2_k}{n_k} \;+\; \frac{\sigma^2_\varepsilon}{n_m \, n_p}},
\end{equation}
where $n_k$ is the product of the sample sizes of the instrumentation facets appearing in index set $k$ (e.g., for interaction $k = \{i,c,m\}$, $n_k = n_mn_in_c$ where $\alpha = ic$ is the object of measurement and thus $n_i=1$ and $n_c=1$).
\end{definition}

Practically, we can look at the join task-transcript $ic$ object of measurement (as the level of unit for a single score outside of developer control), i.e., $S_{\alpha} = \{\alpha \in S : \{i,c\} \subseteq \alpha\}$.

\begin{definition}[Absolute reliability]\label{prop:phi}
The absolute error variance includes all components. The generalizability coefficient (index of dependability) is
\begin{equation}\label{eq:dstudyPhi}
  \widehat{\Phi}_\alpha \;=\; \frac{\sigma^2_{\alpha}}{\displaystyle\sum_{k \,\in\, S} \frac{\sigma^2_k}{n_k} \;+\; \frac{\sigma^2_\varepsilon}{n_m\, n_p}}.
\end{equation}
\end{definition}

\noindent Both coefficients are computed by sampling directly from the posterior of the variance components (rather than from plug-in point estimates), yielding fully Bayesian credible intervals.

\subsubsection{Signal Detectability}

Figures \ref{fig:error_dstudy} and \ref{fig:error_dstudy_phi}, display $E\widehat{\rho}^2$ and $\widehat{\Phi}$ as functions of $n_m$ and $n_p$, showing rapid saturation: even modest numbers of models and prompts suffice to recover the shared error signal. They represent different aspects of the present study's ability to capture meaningful underlying signal across LLMs for cases in $\hat{e}$. In both when $N_{\text{LLM}} = 16$ and $N_{\text{Prompt Strat.}} = 3$, the reliabilities of  $E\widehat{\rho}^2$ and $\widehat{\Phi}$ are reported in Table \ref{tab:reliab_dstudy}. This suggests that the estimated item-transcript scores in this study achieve approximately the target level of consistency expected for these types of data ($0.65$, see \citealt{ho_reliability_2013}).


\begin{table}
\centering
\caption{\textbf{Decision Study}: performing a similar set of estimations to the LLM variance decomposition, except using the human experts on same test sets. We estimate the coefficients for generalizability and dependability for each of the same decompositions, $\hat{e}$, $\hat{e}^\pm$, and $\hat{e}^\pm_{(+)}$ for two units of measurement}\label{tab:reliab_dstudy}
\resizebox{\ifdim\width>\linewidth\linewidth\else\width\fi}{!}{
\begin{tabular}{llrrr}
\toprule
Metric&  Unit of Measure&$\hat{e}$& $\hat{e}^\pm$& $\hat{e}^\pm_{(+)}$\\
\midrule
$E\hat{\rho}^2$&  $\mathcal{C} \times \mathcal{I}$&0.657& 0.837& 0.819\\
 $\widehat\Phi$& $\mathcal{C} \times \mathcal{I}$& 0.551& 0.790&0.781\\
$E\hat{\rho}^2$&  $\{\mathcal{C} \times \mathcal{I}, \mathcal{C} ,\mathcal{I}\}$&0.819& 0.839& 0.946\\
 $\widehat\Phi$& $\{\mathcal{C} \times \mathcal{I}, \mathcal{C} ,\mathcal{I}\}$& 0.732& 0.809&0.931\\
 \bottomrule
\end{tabular}}
\end{table}

\subsubsection{Interpretation}\label{apx:var_decomp_interp}

These results imply that prompt-induced shifts are largely additive and limited in magnitude rather than transformative.  While unexplored prompts may exist, the observed interaction structure suggests that large corrective effects are improbable unless they represent qualitatively new mechanisms outside the sampled strategy universe.  The inference is thus about \emph{effect-size stability}, not prompt exhaustiveness.  In long-context transcript settings, identifying the exact source of brittleness for any single prompt is a Sisyphean task; the generalizability framework instead supports a probabilistic claim: across a realistic and theory-informed prompt universe, prompt choice is unlikely to produce large alignment corrections relative to the error that is shared across all models.

We note that the shared item--observation error is confounded with measurement error in the ``true'' item--classroom score relative to VAM.  Nonetheless, the preponderance of evidence in the present study indicates that the LLM-correlated component of this error is undesirable.  Future work could incorporate hierarchical rater models~\citep{casabianca_hierarchical_2016}, as in \citet{hardy_all_2024}, to estimate the true score as a latent parameter, or apply noise-control methods such as those illustrated in Appendix~\ref{apx:altmethods}.



\paragraph{Prompt Brittleness} We see distributional brittleness in prompting technique in the form of a skewed, long-tailed posterior distribution (see median and upper HDI bound, Tables \ref{tab:sq_err_decomp}, \ref{tab:app_sq_decomp_full}, and Apdx \ref{apx:vardecomp}). 
Practically, this means that certain prompting techniques can \emph{occasionally} inject large amounts of error with no expected contribution. A dramatic improvement from a prompt change would not be expected to generalize to new situations.






\begin{figure}[h!]
    \centering
    \includegraphics[width=1\linewidth]{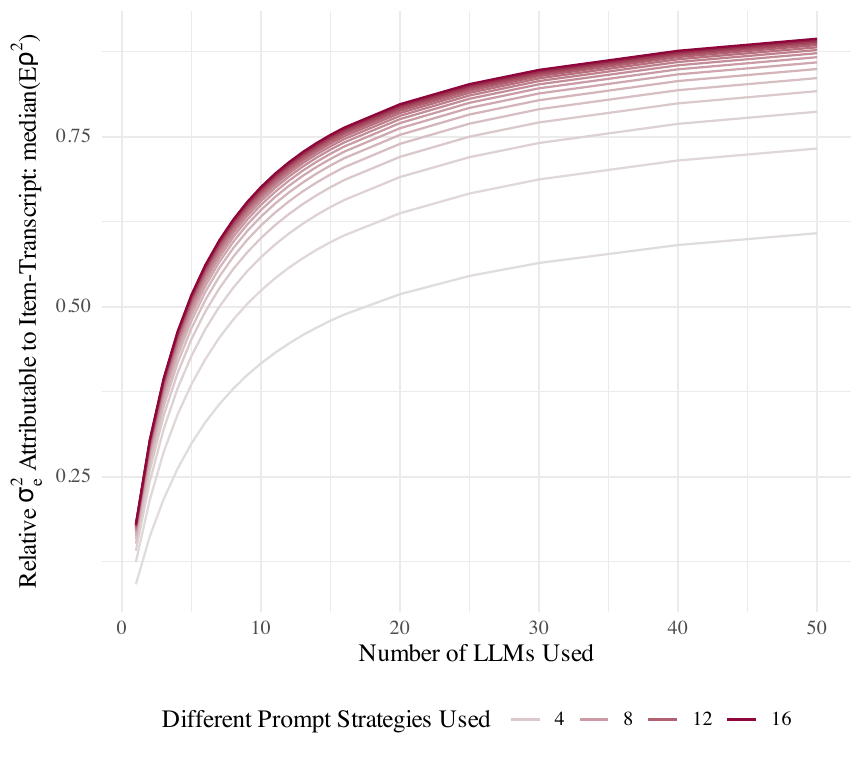}
    \caption{\textbf{Reliability of Relative Shared Error Signal for \textsc{Item}$\times$\textsc{Obs}}: Bayesian Decision Study for item-transcript scores as the object of study across varying numbers of LLMs and Prompting Techniques required to remove LLM and prompt-specific idiosyncrasies from the error signal. Each value is calculated directly from sampling the posterior and extracting the median. For this study, where $N_{\text{LLM}} = 16$ and $N_{\text{Prompt Strat.}} = 3$, the reliability $E\widehat{\rho}^2_{ic}$ is $0.66\pm0.02$ ($\hat{R}=1.00$). This suggests that the estimated item-transcript scores in this study achieve approximately the target level of consistency expected for these types of data ($0.65$, see \citealt{ho_reliability_2013}).}
    \label{fig:error_dstudy}
\end{figure}

\begin{figure}[h!]
    \centering
    \includegraphics[width=1\linewidth]{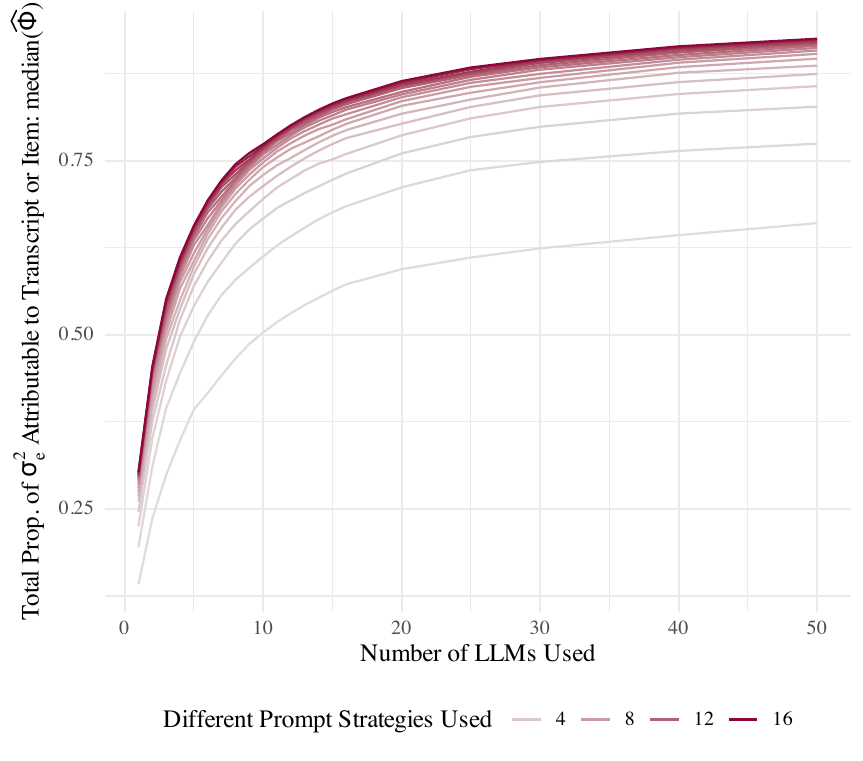}
    \caption{\textbf{Reliability of Absolute Shared Error Signal}: Bayesian Decision Study for any of item, transcript, or their interaction as the objects of measurement ($\sigma^2_{\alpha} =\sigma^2_{i}+\sigma^2_{c}+\sigma^2_{ci}$) across varying numbers of LLMs and Prompting Techniques required to remove LLM and prompt-specific idiosyncrasies from the error signal. Each value is calculated directly from sampling the posterior and extracting the median. For this study, where $N_{\text{LLM}} = 16$ and $N_{\text{Prompt Strat.}} = 3$, the reliability $\widehat{\Phi}$, Equation \ref{eq:dstudyPhi}, of the mean rating across models and prompts shows is 0.73 ($\hat{R}=1.00$).}
    \label{fig:error_dstudy_phi}
\end{figure}

\subsection{Supplementary theory: identifiability and interpretation}\label{apx:theory_vardecomp}

\begin{proposition}[Four-way interaction non-identifiability]\label{prop:nonident}
In a fully crossed $I\times C\times M\times P$ design with one observation per cell, the variance of the four-way interaction term $\nu_{icmp}$ is not separately identifiable from the residual variance.
\end{proposition}
\noindent\textbf{Proof sketch.}
With one observation per cell, any cell-specific deviation from the sum of lower-order effects can be equivalently represented as either a four-way interaction draw or residual noise. The likelihood depends only on their sum, implying non-identifiability. \hfill $\square$

\begin{proposition}[Controllable-variance upper bound]\label{prop:controlbound}
Let $\mathcal{S}_{\mathrm{dev}}:=\{k\in\mathcal{S}: m\in k \ \text{or}\ p\in k\}$ be the set of components involving developer-controlled facets (\textsc{LLM} or \textsc{Prompt}). Then the maximum fraction of error variance that can be affected by changing model and/or prompt (holding the object facets fixed) is upper bounded by
\[
\Pi_{\mathrm{dev}} := \sum_{k\in\mathcal{S}_{\mathrm{dev}}}\pi_k.
\]
\end{proposition}
\noindent\textbf{Interpretation.}
$\Pi_{\mathrm{dev}}$ is an \emph{optimistic} ceiling: it counts as “controllable” even those interactions (e.g., \textsc{LLM}$\times$\textsc{Obs}) whose direction may not generalize across new classroom segments. A small $\Pi_{\mathrm{dev}}$ therefore implies that prompt/model changes cannot be expected to deliver large, reliable alignment improvements.

\begin{corollary}[When prompt engineering is predictably effective]\label{cor:prompt}
If $\pi_{p}$ and $\sum_{k: p\in k}\pi_k$ are both large and have concentrated posteriors (narrow HDIs), then prompt changes can be expected to yield reliable error reductions across contexts; conversely, small medians with wide HDIs indicate brittle prompt behavior.
\end{corollary}

\subsection{Parallel estimation on human expert error}\label{apx:humanvar}
 
To aid readers less familiar with interpreting variance decompositions, we provide a parallel analysis that replaces the \textsc{LLM}$\times$\textsc{Prompt} instrumentation facets with individual human rater identifiers.  This is a \emph{different} estimation from Eq.~\ref{eq:vardecomp_full}---the facets, and therefore what counts as ``controllable'' versus ``systemic,'' change accordingly.  The results (Table~\ref{tab:hum_err_decomp}) are not directly comparable in absolute terms, but the \emph{proportional} distribution of variance is instructive.
 
\subsubsection{Expected pattern for expert raters}
 
If human raters are performing the rating task competently, and if items vary in their contribution to VAM~\citep{kane_have_2013, hardy_all_2025}, then most observed human error should concentrate on the item facet and item--observation interaction.  This is precisely what we find.
 
\subsubsection{Results}
 
The rater main effect and the observation main effect each account for 0\% of variance in squared error---individual expert biases are negligible in the absence of interactions.  For unsigned error, $\hat{e}$ he item main effect absorbs 33.8\%, and the item--observation interaction an additional 25.2\%, for a combined 59.0\% attributable to the rubric dimension and its interplay with the specific lesson segment.  This is the region where ``true score'' mismatch between expert ratings and VAM is most expected~\citep{kane_have_2013}.  Rater-specific biases on particular items or lesson segments contribute a combined 17.3\%.

This finding becomes much more pronounced when sign is included in the error signal $\hat{e}^\pm$, and $\hat{e}^\pm_{(+)}$. Here, the combined variance in error attributable to the item alone increases to 69.3\%. If the items themselves are misaligned, we would expect humans to be highly correlated and sensitive to the addition of such items. We also see this dramatic variability in the LLMs with the addition of the three additional items.
 
\subsubsection{Contrast with LLM error structure}
 
For LLMs, the analogous item-plus-item$\times$observation share is only 15\% of total error variance.  The error structure of expert humans is thus concentrated where theory predicts it should be---on the constructs being measured---whereas LLM error disperses across high-order interactions involving model and prompt, indicating instability rather than construct-related variation. Proportions absorbed by the residuals are similar between LLMs and humans, so the variation in shares is comparable.

\begin{table}
\centering
\caption{\textbf{Human Error Variance Decomposition}: performing a similar set of estimations to the LLM variance decomposition, except using the human experts on same test sets. We estimate the percentage of human expert-VAM error for each of the same decompositions, $\hat{e}$, $\hat{e}^\pm$, and $\hat{e}^\pm_{(+)}$ using Restricted Maximum  Likelihood with \texttt{lmer} \citep{bates_fitting_2015}.}\label{tab:hum_err_decomp}
\resizebox{\ifdim\width>\linewidth\linewidth\else\width\fi}{!}{
\begin{tabular}{lrrr}
\toprule
Facet & Pct($\hat{e}$) & Pct$(\hat{e}^\pm)$& Pct$(\hat{e}^\pm_{(+)})$\\
\midrule
ITEM & 33.76 & 63.75 & 69.27\\
OBS & 0.00 & 0.00 & 1.10\\
RATERID & 0.00 & 0.00 & 0.00\\
\addlinespace
ITEM:OBS  & 25.22 & 13.73 & 11.84\\
RATERID:OBS & 12.85 & 8.01 & 3.69\\
RATERID:ITEM & 4.42 & 3.27 & 3.35\\
RATERID:ITEM:OBS + $\epsilon$ & 23.76 & 11.24 & 10.75\\
\bottomrule
\end{tabular}}
\end{table}

The code for the estimations of the models can be found in code Listing \ref{listing:human}
\begin{lstlisting}[language=R] 
lme4::lmer(errSCORE ~ 
         (1|ITEM) 
       + (1|OBS_CHAP)
       + (1|RATERID)
       + (1|ITEM:OBS_CHAP)
       + (1|RATERID:ITEM)
       + (1|RATERID:OBS_CHAP), 
       data = df, 
control=lmerControl(optimizer="bobyqa"))
\end{lstlisting}\label{listing:human}

\section{ Measures and Metrics}\label{apx:metrics}



\subsection{Measuring Dependence with Distance Correlation}\label{app:dcor}

In this appendix, we provide a self-contained technical overview of the bias-corrected squared distance correlation, the primary statistic used in our analysis to quantify the dependence between raters. We first motivate the need for a measure beyond traditional linear correlation and then formally define the statistic and its properties.

\subsubsection{Motivation: Beyond Linear Correlation}
A common method for measuring the association between two random variables, $U$ and $V$, is the Pearson product-moment correlation coefficient, $\rho(U,V)$. While widely understood, Pearson's $\rho$ is designed to detect only \textit{linear} relationships. It can be zero even when the variables share a strong, deterministic, but nonlinear relationship (e.g., $V = U^2$ for a symmetric $U$).

In the context of this study, we are comparing rating distributions from different sources (LLM-LLM, LLM-human). There is no \textit{a priori} reason to assume that the relationship between two sets of ratings is linear. For example:
\begin{itemize}[nosep]
    \item One model might use a compressed range of scores compared to another, leading to a curvilinear relationship.
    \item Two raters might agree on clear-cut cases (very high or very low quality instruction) but diverge in their assessment of moderately effective instruction, producing a non-monotonic relationship.
\end{itemize}
Relying on a linear measure in such cases would systematically underestimate the true degree of behavioral correspondence between the raters.

To overcome this limitation, we employ the \textbf{distance correlation}, a non-parametric measure of dependence introduced by \citet{szekely_measuring_2007}. Distance correlation is designed to capture any type of statistical dependence between two random vectors of arbitrary and not necessarily equal dimension. Its central property, which makes it exceptionally powerful, is that the population distance correlation is zero \textit{if and only if} the variables are statistically independent.

\subsubsection{The Bias-Corrected Squared Distance Correlation}
We now formally define the squared distance correlation and its bias-corrected sample estimator, which we denote $\operatorname{dCor}^2_n$. We focus on the squared value, as it is more convenient for statistical inference and bias correction.

Let $(X, Y) = \{(x_k, y_k) : k=1, \dots, n\}$ be a statistical sample of $n$ paired observations from a joint distribution $(U,V)$. In our case, $x_k$ and $y_k$ represent the scalar ratings given by two different raters to the $k$-th classroom transcript.

\paragraph{Sample Distance Matrices}
The computation begins by constructing Euclidean distance matrices for each variable. Let the $n \times n$ distance matrix for the $X$ sample be $\mathbf{a}$, with entries:
\begin{equation}
    a_{kl} = \|x_k - x_l\| = |x_k - x_l|,
\end{equation}
and similarly for the $Y$ sample, $\mathbf{b}$, with entries $b_{kl} = |y_k - y_l|$. The use of absolute differences is a specific case of the Euclidean norm for scalar ratings.

\paragraph{Double Centering and Sample Distance Covariance}
To make the measure invariant to rigid transformations (translation and orthogonal rotation), the distance matrices are \textit{double-centered}. For each element $a_{kl}$ in the distance matrix $\mathbf{a}$, we compute its centered counterpart $A_{kl}$:
\begin{equation}
    A_{kl} = a_{kl} - \bar{a}_{k\cdot} - \bar{a}_{\cdot l} + \bar{a}_{\cdot \cdot},
\end{equation}
where $\bar{a}_{k\cdot} = \frac{1}{n}\sum_{l=1}^n a_{kl}$ is the $k$-th row mean, $\bar{a}_{\cdot l} = \frac{1}{n}\sum_{k=1}^n a_{kl}$ is the $l$-th column mean, and $\bar{a}_{\cdot \cdot} = \frac{1}{n^2}\sum_{k,l=1}^n a_{kl}$ is the grand mean of the distance matrix. The same procedure is applied to $\mathbf{b}$ to obtain the centered matrix $\mathbf{B}$.

The \textit{sample squared distance covariance}, $\operatorname{dCov}^2_n(X,Y)$, is the arithmetic average of the products of the corresponding centered distances:
\begin{equation}
    \operatorname{dCov}^2_n(X,Y) = \frac{1}{n^2} \sum_{k,l=1}^n A_{kl} B_{kl}.
\end{equation}
The sample squared distance variances are the distance covariances of each variable with itself: $\operatorname{dCov}^2_n(X,X)$ and $\operatorname{dCov}^2_n(Y,Y)$.

\subsubsection{The Bias-Corrected Estimator}
The natural (but biased) sample estimator for the squared distance correlation is the ratio of the sample squared distance covariance to the product of the sample distance standard deviations. However, this estimator is positively biased for finite samples; it will be non-zero on average even for independent variables. This bias complicates inference, especially with moderate sample sizes.

To address this, \citet{szekely_partial_2014} introduced a bias-corrected estimator, which we denote $\operatorname{dCor}^2_n(X,Y)$. It is based on a U-statistic estimator of the squared distance covariance, which is unbiased. The computation involves a slightly different centering:
\begin{align}
    \tilde{A}_{kl} &= a_{kl} - \frac{1}{n-2}\sum_{j=1}^n a_{kj} - \frac{1}{n-2}\sum_{i=1}^n a_{il} \\ \nonumber &+ \frac{1}{(n-1)(n-2)}\sum_{i,j=1}^n a_{ij}.
\end{align}
The unbiased estimator of squared distance covariance is then:
\begin{equation}
\label{eq:unbiased_dcov}
    \widetilde{\operatorname{dCov}}^2_n(X,Y) = \frac{1}{n(n-3)}\sum_{k \neq l} \tilde{A}_{kl} \tilde{B}_{kl}.
\end{equation}
The bias-corrected squared distance correlation, $\operatorname{dCor}^2_n$, is constructed as the ratio of these unbiased estimators:
\begin{equation}
    \operatorname{dCor}^2_n(X,Y) = \frac{\widetilde{\operatorname{dCov}}^2_n(X,Y)}{\sqrt{\widetilde{\operatorname{dCov}}^2_n(X,X) \widetilde{\operatorname{dCov}}^2_n(Y,Y)}}.
\end{equation}
Due to the bias correction, $\operatorname{dCor}^2_n$ can take small negative values when the true dependence is zero or near-zero. In practice, negative estimates are typically treated as evidence of independence and can be reported as zero. The statistic is bounded above by 1.

\subsubsection{Key Properties}
The power of distance correlation is summarized in the following from \citet{szekely_measuring_2007}:
Let $U$ and $V$ be random variables with finite first moments. The population distance correlation $\operatorname{dCor}(U,V) = 0$ if and only if $U$ and $V$ are statistically independent.
This property guarantees that distance correlation will detect any departure from independence, linear or not. In contrast, Pearson correlation's analogous property holds only for bivariate normal distributions. By using the bias-corrected estimator $\operatorname{dCor}^2_n$, we obtain a reliable and sensitive measure of the dependence between rater judgments, robust to the specific functional form of their relationship and suitable for the sample sizes in our study. This makes it the ideal tool for uncovering the subtle but strong patterns of behavioral convergence among LLMs documented in Section \ref{sec:discdcor}.

\begin{figure*}[htbp]
    \centering
    \includegraphics[width=1\linewidth]{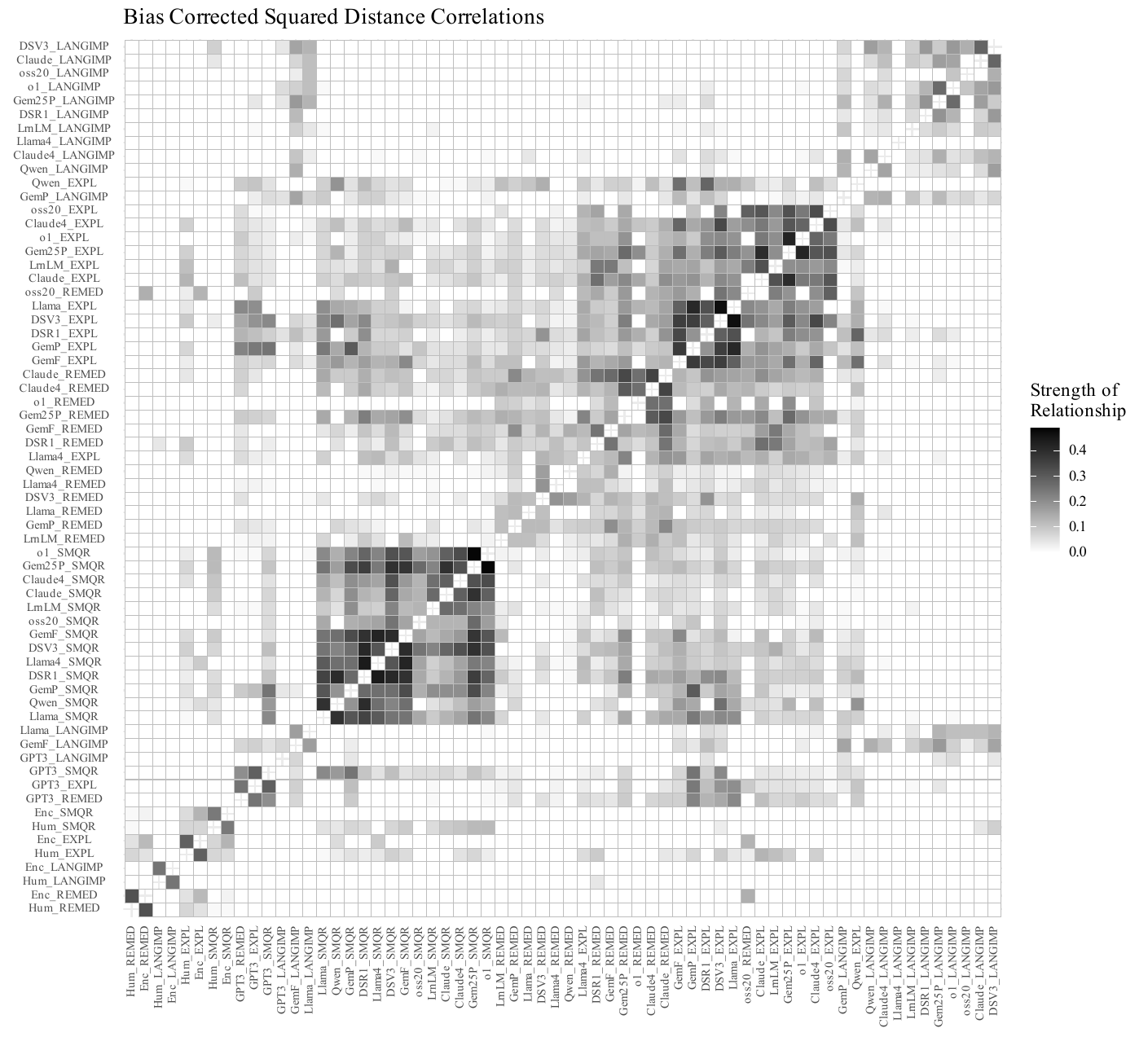}
    
    \caption{\textbf{Between and Within LLM Distance Correlations: MQI}, distance correlations nonparametric measure of dependence between and within rater families across MQI items. Correlation are conducted at the item-transcript level using pairwise-complete observations. Nonsignificant relationships (at $\alpha < 0.05$) are shown as blank after adjusting for family-wise error rate using the Bonferroni correction. Hierarchical clustering is done using complete linkages.}
    \label{fig:distance}
\end{figure*}

\begin{figure*}[htbp]
    \centering
    \includegraphics[width=1\linewidth]{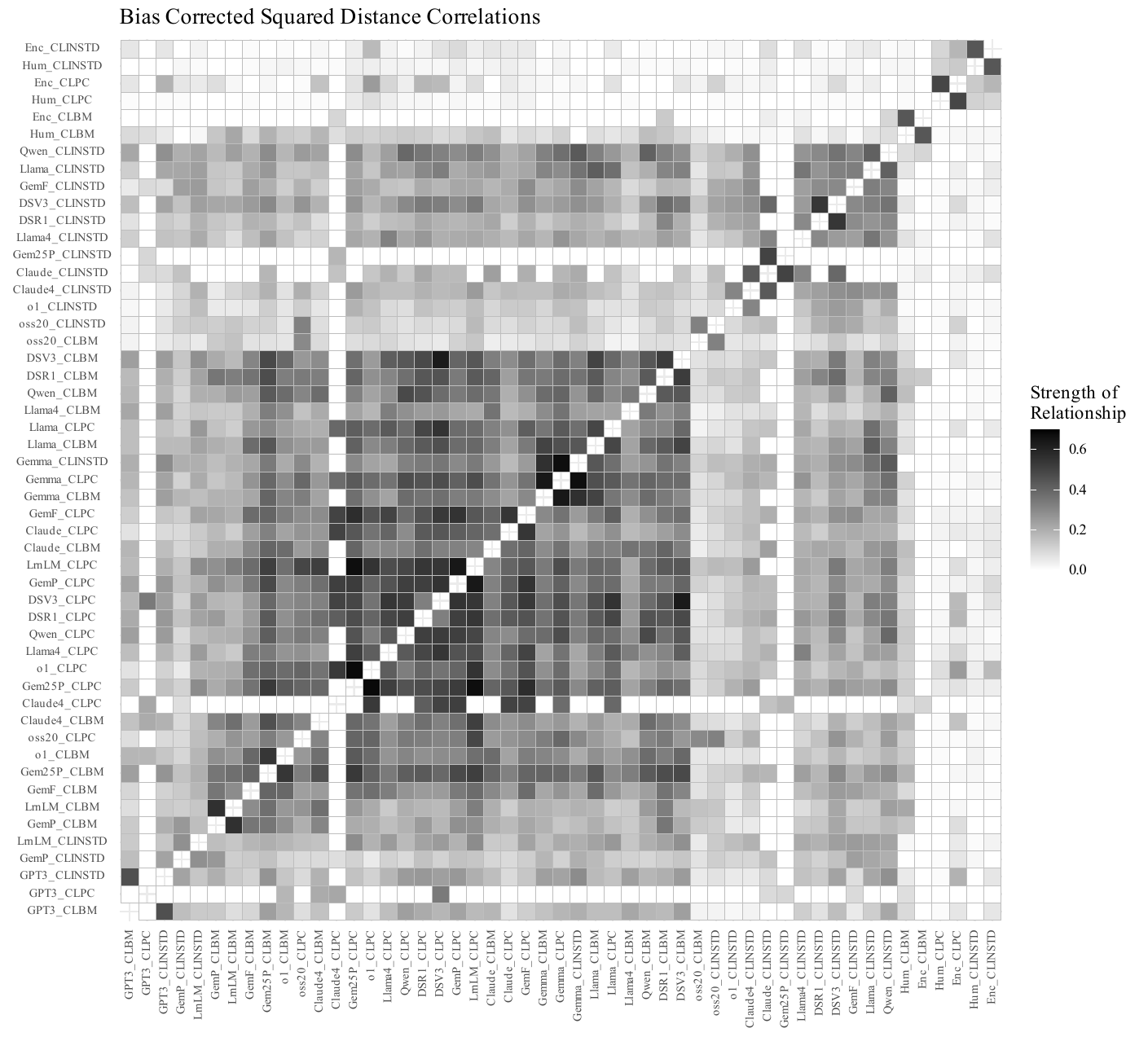}
    
    \caption{\textbf{Between and Within LLM Distance Correlations: CLASS}, distance correlations nonparametric measure of dependence between and within rater families across CLASS items. Correlation are conducted at the item-transcript level using pairwise-complete observations. Nonsignificant relationships (at $\alpha < 0.05$) are shown as blank after adjusting for family-wise error rate using the Bonferroni correction. Hierarchical clustering is done using complete linkages.}
    \label{fig:distance_class}
\end{figure*}

\subsection{Disattenuated Stacked Correlations for Underlying VAM}\label{apx:disatt}
By stacking both VAM measures (STA and ALT in Eq. \ref{eq:taucorrect}) \citep{kane_estimating_2008,kane_gathering_2012} at the teacher-year level, we can estimate the correlation of the underlying value-added to student learning through these two standardized measures. The correlation between the two VAMs is reduced by the fact that there is measurement error in both estimates for a given teacher-year.  

We can use the noise-ceiling, the geometric mean of the product of implied reliabilities with the underlying value add, to disattenuate the correlation from some of the measurement error. This is because the correlation between either measure and the underlying value-add is the square root of the correlation between the two noisy measures \cite{kane_gathering_2012}.

\begin{equation}\label{eq:taucorrect}
    \tau_{S_{j},{Y}} = \frac{\tau_{S_j\tilde{Y}}}{\sqrt{\tau_{\tilde{Y}_{{STA}},\tilde{Y}_{{ALT}}}}}
\end{equation}

We use this relationship when measuring the correlations between classroom ratings and the underlying teacher value-add on student learning. This scalar transform for the y-axis has no effect on positions (only changing the y-axis tick marks) in Figure \ref{fig:alignment}. The findings of the study are robust to (and would be strengthened by) additionally utilizing Greiner's equality, $\rho =\sin{\left(\frac\pi 2 E[\tau]\right)}$ \citep{greiner_ueber_1909} when performing this transformation. We report the results without this transformation both for simplicity and to better preserve the alignment nature of $\tau$.

\subsection{Expert Ensembling}\label{sec:ensemble-failure}

Conventional wisdom suggests that ensembling multiple models improves robustness and accuracy by leveraging diverse model strengths or averaging out independent errors. Our findings directly challenge this assumption for educational evaluation tasks. We tested two conceptually opposed ensemble strategies: (1) \textbf{pedagogy-expertise weighting}, which weights model votes by performance on an AI pedagogy benchmark \citep{lelievre_benchmarking_2025},\footnote{Findings are robust to using other expertise benchmark scores such as MMLU-Pro and math-specific benchmarks} and (2) \textbf{unanimous voting}, which selects only cases where all models agree, distilling their shared consensus.

For a set of models $\mathcal{F} = \{f_1, \ldots, f_K\}$, the pedagogy-weighted ensemble score for transcript $t$ on task $j$ is:
\begin{equation}
S_{\text{weighted}, tj} = \frac{\sum_{k=1}^K w_k \cdot S_{f_k, tj}}{\sum_{k=1}^K w_k}
\end{equation}
where $w_k$ represents model $f_k$'s performance on a pedagogy knowledge benchmark (we test MMLU Education subset, MMLU Mathematics, and a specialized mathematics pedagogy benchmark). The unanimous ensemble restricts analysis to the subset $\mathcal{T}_{\text{unan}} = \{t : S_{f_1,tj} = S_{f_2,tj} = \cdots = S_{f_K,tj}\}$ of transcripts where all models assign identical ratings.

Figure \ref{fig:alignment} (middle and bottom rows) reveals that \textbf{both ensemble strategies not only fail to improve alignment with student learning but dramatically worsen it} for several critical instructional dimensions. For \texttt{REMED} (remediation of student errors), pedagogy-weighted ensembles (middle row) shift the alignment distribution downward: median $\tau_{S_{\text{weighted}}Y}$ decreases from approximately $-0.15$ (individual models) to $-0.28$ (weighted ensemble), a statistically significant degradation ($p < 0.01$, bootstrap test). Similarly, for \texttt{CLBM} (behavior management), unanimous voting ensembles (bottom row) shift dramatically lower $\tau_{S_{\text{unan}}Y} < -0.2$, compared to the more dispersed individual model distribution.

\section{Alternate Methods for Downstream Task Alignment: BLUPs}\label{apx:altmethods}
\subsection{Teacher Skill Model Best Linear Unbiased Predictions (BLUPs)}\label{apx:rater}

A teacher skill model described allows us to capture the nested structure of the data and account for various sources of random variation, including teacher effects, rater bias, lesson-specific factors, and skill-specific effects. By modeling these sources of variance as random effects, we obtain more accurate estimates of the underlying teacher skill. The strength of this approach is that it helps isolate the parts of the variation attributable to specific teacher skill targeted by the rubric item, which is particularly important where there may be correlations across skills.  

The detailed model specification is given by:

\begin{align}
\text{S}_{rjsli} &=  \mu + 
\nu_{i} + 
\nu_{r} + 
\nu_{j} + 
\nu_{l:i} + 
\nu_{s:l:i} + \nonumber  \\ &
\nu_{ji} + 
\nu_{jl:i} + 
\nu_{js:l:i} + 
\nu_{ri} 
+ \nu_{rl:i} 
+ \nu_{rs:l:i} \nonumber  \\ &
+ \nu_{rji} 
+ \nu_{rjl:i} 
+ \nu_{rj}
+ \epsilon_{rjs:l:i}
\label{eq:gstudy}
\end{align}

where $S_{rjsli}$ represents the score for Item $j$ given by Rater $r$, during class segment $s$ within lesson $l$ taught by teacher $i$. Here, $\mu$ is the overall mean, and the random effects capture the hierarchical variance. Our focal component for evaluating model annotations at the classroom transcript level is $\nu_{js:l:i}$, which encapsulates the variation attributable to the observable teacher skill for each segment after accounting for other sources of variation. This allows us to estimate BLUPs for the individual transcript sections seen by the models, thereby isolating the variation associated with the specific task of interest, $\hat{\nu}_{js:l:i}$, which is used in the correlation analyses. A breakdown of the notation for the model is below:

\begin{itemize}[nosep]
    \item $X_{rjsli}$ represents the rating for the $i$-th teacher, by the $r$-th rater,  on the $j$-th skill, during the $s$-th segment, for the $l$-th lesson.
    \item $\mu$ is the overall grand mean.
    \item $\nu_i$ is the random effect for the $i$-th teacher. Interpretation: some teachers receive higher ratings than others.  \footnote{For holistic evaluations, this might be considered the ``true score'' of a teacher's overall performance.}
    \item $\nu_r$ is the random effect for the $r$-th rater: some raters are more lenient than others.
    \item     $\nu_j$ is the random effect for the $j$-th skill item: some items are easier than others.
    \item $\nu_{rj}$ is the random effect for the interaction between rater and skill. Some raters score certain items higher.
    \item $\nu_{l:i}$ is the random effect for the $l$-th lesson within a teacher: some lessons receive higher ratings than others. Confounded with teacher overall score dependence on lessons.
    \item $\nu_{s:l:i}$ is the random effect for the $s$-th segment of the $l$-th lesson for the $i$-th teacher. A segment is the unit of one transcript: some lessons receive higher ratings than others. Confounded with lesson overall score dependence on segments.
    \item $\nu_{ji}$ is the random effect for the interaction between the skill and teacher: some teachers score higher on some skills. This component will be used to estimate BLUPs of the expected skill level for a teacher. 
    \item $\nu_{jl:i}$ is the random effect for the interaction between skill and lesson: some lessons score higher on some skills. 
    \item $\nu_{js:l:i}$ is the random effect for the interaction between skill and segment: some segments of lessons receive higher scores than others. This component will be used to estimate BLUPs at the transcript level. 
    \item $\nu_{ri}$ is the random effect for the interaction between rater and teacher: some raters score certain teachers higher.
    \item $\nu_{rl:i}$ is the random effect for the interaction between rater and lesson. Some raters score certain lessons higher.
    \item $\nu_{rji}$ is the random effect for the interaction between rater, skill, and teacher: some raters score certain teachers higher on some items.
    \item $\nu_{rjli}$ is the random effect for the interaction between rater, skill, and lesson: some raters score certain lessons higher on some items.
    \item $\epsilon_{rjsli}$ is the residual error term.
\end{itemize}

\subsubsection{Code Listing}
Below is the code used for the estimation of the model defined by Equation \ref{eq:gstudy} using \texttt{lme4} syntax. The model was estimated using restricted maximum likelihood (REML) with the BOBYQA optimizer. The variable names are the same found in the original study \cite{kane_national_2015}. Scores were minmax scaled to $[0,1]$ prior to estimation.

\begin{lstlisting}[language=R] %[htb]
% \begin{minted}{R}
lmer(SCORE ~ (1 | NCTETID)  + (1 | ITEM) 
+ (1 | RATERID) + (1 | RATERID:ITEM) 
+ (1 | RATERID:NCTETID) 
+ (1 | RATERID:OBSID:CHAPNUM) 
+ (1 | OBSID/CHAPNUM) 
+ (1 | OBSID:CHAPNUM:ITEM) 
+ (1 | OBSID:ITEM) 
+ (1 | NCTETID:ITEM) 
+ (1 | RATERID:ITEM:NCTETID) 
+ (1 | RATERID:OBSID:ITEM) 
+ (1 | RATERID:OBSID), data = df, 
control=lmerControl(optimizer="bobyqa"))
\caption{Target Teaching Skill Estimation Model Code}
\label{code:lmefit}
\end{lstlisting}

\paragraph{BLUP Extraction}
BLUPs were extracted \texttt{lme4::predict(model, re.form = ~ (1 | NCTETID:ITEM))} and \texttt{lme4::predict(model, re.form = ~ (1 | OBSID:CHAPNUM:ITEM))} for the respective $\nu_{ji}$ and $\nu_{js:l:i}$ 

\paragraph{Parameter Estimates}
Parameter estimates for the model are in Table \ref{tab:random_effects}.
\begin{table*}[htbp]
    \centering
    \caption{Mixed Effect Estimates for Teacher Skill Model}
    \label{tab:random_effects}
    \resizebox{0.85\textwidth}{!}{
        \begin{tabular}{lllrcc}
            \toprule
            \textbf{Parameter} & \textbf{Group} & \textbf{Effect Type} & \textbf{Coefficient} & \textbf{CI Low} & \textbf{CI High} \\
            \midrule
            (Intercept) & --- & Fixed & 0.46 & 0.33 & 0.58 \\
            \midrule
            Segment & $w_{k}$ & Random ($\sigma^2_{k}$) & 0.03 & -- & -- \\
            \cmidrule(lr){2-6}
            & RATERID:OBSID:ITEM & Random & 0.09 & -- & -- \\
            & RATERID:ITEM:NCTETID & Random & 0.03 & -- & -- \\
            & RATERID:OBS\_CHAPS & Random & 0.03 & -- & -- \\
            & RATERID:ITEM & Random & 0.05 & -- & -- \\
            & RATERID:OBSID & Random & 0.04 & -- & -- \\
            & RATERID:NCTETID & Random & 0.02 & -- & -- \\
            & OBSID & Random & 0.02 & -- & -- \\
            & Residual & Random & 0.16 & -- & -- \\
            & NCTETID & Random & 0.03 & -- & -- \\
            & RATERID & Random & 0.03 & -- & -- \\
            & ITEM & Random & 0.32 & -- & -- \\
            \bottomrule
        \end{tabular}
    }
\end{table*}

\section{Alternate Methods for Intended Impact Alignment: Refined VAM Residualization Model}\label{apx:vam}
\subsection{Refining VAM Measurements}
For the y-axis, we measure the alignment between transcript ratings and end of year value-added measures. To assess the alignment between Large Language Model (LLM) ratings of teacher skills and student learning gains, we employ teacher value-added measures (VAMs) derived from standardized assessments.  However, VAMs operate at the teacher or teacher-year level, while LLM ratings are generated at the finer-grained lesson-segment-by-teacher-skill level.  

Naïve aggregation of LLM ratings would be inappropriate due to the sparse and unbalanced nature of the data, as LLM evaluations were only available for short segments of lessons and for a subset of teacher skills.  To bridge this gap in granularity and isolate the relationship between LLM ratings and VAMs, we employ a two-stage approach involving a mixed-effects model and subsequent semipartial correlation analysis.

The residuals from this model, which represent the VAMs after accounting for confounding variables, are then used to compute semipartial correlations. We opted for semipartial correlations to ensure that the controls applied to the dependent variable (VAMs) remain constant across all LLMs, facilitating fair comparisons. 

To remove confounding variance from the VAM signal, we construct a comprehensive mixed-effects model to partition the variance in teacher VAMs, accounting for 1) variables that would contribute to VAM scores that are unrelated to a teacher's instructional practice and 2) variables that would mediate the relationship between an observer's score and VAM. This model allows us to isolate the residual variance in VAMs specifically attributable to observable instructional skills, along with any remaining measurement error. By residualizing the VAMs, we remove variation attributable to the representativeness of the lesson segment, and remove confounds potentially introduced by our decision to model with all the information (various VAMs and all 25 items), and account for the various programmatic, curricular, student population, and school-year related relationships to 

This methodological approach addresses several key challenges:

1. \textbf{Granularity mismatch}: By using lesson-segment level residuals, we preserve the fine-grained nature of LLM ratings while relating them to year-level VAMs.

2.\textbf{ Unbalanced observations}: Our approach accounts for the fact that LLM ratings typically cover only one section of a lesson, with little overlap between teachers observed.

3. \textbf{Preservation of meaningful variance}: The random effect for teachers, scaled by the difference between observed and expected performance, retains more variance for teacher observations that are closest to their typical performance.

4. \textbf{Ordinal nature of ratings}: The use of Kendall correlations acknowledges the ordinal nature of most teaching quality ratings and focuses on directional alignment rather than precise numeric agreement.

5. \textbf{Multiple sources of VAMs}: By incorporating multiple types of VA measures, we account for differences in assessment season relative to classroom observations.

\subsection{Value-added Measures}
Two complementary value-added measures are used at two levels of aggregation: VAMs calculated from formal state standardized assessments and another from study-specific, low-stakes and informal assessments, aggregated at the teacher-year level and across years. Although imperfect,  together the VAMs provide a more robust measure with the underlying teacher value-add.  The residualization model is estimated using all information from available VAM types.  Only VAM residuals corresponding to the same teacher-year as observed lesson transcript are used for the subsequent semipartial correlations.  

By stacking VAMs from both the state and the study-specific exams, we can correlate rater and LLM scores against the underlying teacher value-add as measured through these instrument. We can then use the correlation of these VAMs to correct for rater correlations against this underlying value-add.

\subsubsection{Controls for Confounding}
The analysis further controls for numerous confounding factors, including district, season of observation, grade level, class composition, teacher experience, class size, and subject-specific expertise, to isolate the effect of instructional quality on student outcomes.

\subsubsection{Controls for Divergence from Typical Teacher Skill}
To address the issue of differing levels of granularity, we leverage Best Linear Unbiased Predictors (BLUPs) extracted from the fully crossed random effects model described in the previous section (and explicitly defined in Eq, \ref{eq:gstudy}) to isolate the signal most relevant to interpreting the relationships between ratings of instructional quality and student learning.  

Specifically, we estimate two BLUPs with this model. The first, $\hat{\nu}_{jsli}$, was used in the previous section and estimates the skill level displayed by a teacher during a specific segment of a lesson.  The second, $\hat{\nu}_{ji}$, estimates the latent ability for a specific skill for a teacher across all observations. The difference between these, $\Delta_{si,j}= \hat{\nu}_{js:l:i} - \hat{\nu}_{j:i}$, provides a metric that captures the \textit{representativeness} of a particular lesson segment in relation to the teacher's overall skill profile. This difference is then used as a predictor in a linear mixed-effects model to adjust the VAMs for the representativeness of each lesson segment. 

The residualizing model is stylized as follows:

\begin{align*}
V_{dgs:l:ityvj} &= \beta_1 D_d + \beta_2 G_g + \beta_3 (D_d G_g) \\
&+ \boldsymbol{\beta_C} \mathbf{C}_{cgSyt} + \boldsymbol{\beta_S} \mathbf{S}_{dSyi} \\ &+ \boldsymbol{\beta_T} \mathbf{T}_{i} + \tau {\boldsymbol{M}}_{s:l}  
\\ &+  \boldsymbol{\nu}_{dglsv} [D_d \times (G_g \times Y_y  +  m_{s:l})] \\ &+  \boldsymbol{\nu}_{dglsv} [D_d \times (G_g \times Y_y  +  m_{s:l}) \times v ] \\
&+ \boldsymbol{\nu}_{dglsv} S_S+ \gamma_t\Delta_{si,j}  + \epsilon_{dgsyt}
\end{align*}
where:
\begin{itemize}[nosep]

    \item  $V_{dgsytj}$ represents the stacked teacher VAM for district $d$, grade $g$, school $S$, year $y$, and term $t$, incorporating multiple standardized assessment outcomes.  Stacking multiple VAMs increases the reliability of the overall teacher effectiveness measure and mitigates potential biases associated with the season of assessments relative to classroom observations.

    \item   $D_d$ fixed effects systematic differences in student achievement attributable to District- programmatic and curricular policies and practices.  The inclusion of the interaction term acknowledges the potential for district-specific effects by grade-level $G_g$ for curricular and programming decisions.

    \item  $\mathbf{C}_{cgSyt}$, $\mathbf{S}_{dSyt}$, and $\mathbf{T}_{i}$ represent vectors of class-, school-, and teacher-level covariates, respectively. These include prior achievement, student demographics, class size, school size, teacher experience, and measures of teacher knowledge.  These covariates control for factors that influence student achievement that would not be observable by raters of instruction or are not directly related to the teacher instructional quality.

    \item  $\mathbf{M}_{s:l}$, vectors of  lesson segment ($s$) within the lesson ($l$)-specific covariates with respect to the time during the school year, length of class, and the unrepresentativeness of the observation with respect to the teacher's average expected level of performance, $\Delta_{si,j}$.
    \begin{itemize}
        \item  $\Delta_{si,j}$ represents the difference between the teacher's observed skill level in a specific lesson segment ($\hat{\nu}_{jsli}$) and their overall mean skill level ($\hat{\nu}_{ji}$), calculated as $\Delta_{si,j}= \hat{\nu}_{jsli} - \hat{\nu}_{ji}$.  These Best Linear Unbiased Predictors (BLUPs) are derived from a separate, fully crossed mixed-effects model in Equation \ref{eq:gstudy} and incorporated here to weight each lesson segment according to its representativeness of the teacher's overall skill profile.  This approach allows us to preserve the variance contributed by lesson segments that are most indicative of the teacher's typical performance, supporting observation-level inference. 
    \end{itemize}
    \item  Random effects, $\boldsymbol{\nu}_{dgltsv}$, are crossed and indexed by district, grade, time during the school year, type of value-added measure  $\boldsymbol{\nu}_{dglsv} [D_d \times (G_g \times Y_y  +  m_{s:l}) \times ( 1+ v )   + S_S]$, account for the nested structure of the data and the potential correlations within districts and between different types of VAM ($v$).  Additionally, each school $S_S$ has a random slope to account for variance not attributable to the fixed effects. It is important to not use a school-level random effect since a fixed effect would distort the residuals, for example, in schools where all teachers were good
    \item The random slopes for each teacher $\gamma_t$ are with respect to $\Delta_{si,j}$,  allows for teacher-specific variation in the relationship between lesson segment representativeness and VAM, preserving the variation for those segments that are more representative of the teacher.

    \item  $\epsilon_{dgs:l:ityvj}$ represents the residual error, which retains the 1) variation in VAM attributable to each lesson segment, scaled by its representativeness vis-a-vis the teacher, whose effects on VAM are lessened as the segment diverges from the teacher's average, 2) measurement error, and 3) any remaining bias from omitted variables.

\end{itemize}

We represent the residualized VAM values for an individual teacher $\epsilon_{dgs:l:ityvj} = \tilde{V}_{sjv}$, for each unique nested segment-skill-VAM intersection. For use in the semi-partial correlations of the main portion of this study, we use only the subset of the residuals corresponding to the end of year VAMs on the state assessment and alternative assessment (original study) $v \in \{STA, ALT \}_y$, to better align teaching practices to same-year outcomes.

Below is the code listing for the model specification using the \texttt{lme4} software in \texttt{R}.

\begin{lstlisting}[language=R]

outcome ~ 0
       + DISTRICT*GRADE ## District fixed effects and grade fixed effect baselines
       + V_CS_ALT_IRT_M_TM1 ## class mean standardized student performance on BOY assessment
       + V_CS_STATE_STD_M_TM1 ## class mean standardized student performance on previous year math assessment
       + V_CS_STATE_STD_E_TM1 ## class mean standardized student performance on previous year English/reading assessment
       + V_CCLASS_SIZE ## class size
       + MAXCHAP ## length of class observed
       + V_CS_SPED ## class prop. of SPED students
       + V_CS_LEP  ## class prop. of English Learners
       + V_CS_FRPL ## class prop. of Low-SES students
       + ACCURACY_yr_est ## standardized measure of teacher accuracy in predicting their student's errors
       + KOSM_yr_est ## standardized measure of teacher knowledge of common student misconceptions
       + MATH_KNOWLEDGE_yr_est ## standardized measure of teacher knowledge of mathematics for instruction
       + RepDelta ## difference between the BLUP-estimated scores of this observation and BLUP-estimated teacher average for that observational item
       + TIMING ## season/month during the year of the observation
       + V_SS_FRPL ## school prop. of Low-SES students
       + V_SS_SPED ## school prop. of SPED students
       + V_SS_LEP ## school prop. of English Learners
       + V_SS_STATE_STD_M_TM1 ## school mean standardized student performance on previous year math assessment
       + V_SS_STATE_STD_E_TM1 ## school mean standardized student performance on previous year English/reading assessment
       + V_SSCHOOL_SIZE ## school size
       + (1|DISTRICT:GRADE:SCHOOLYEAR_SP) ## variations in programming, resources, and capacity in districts by grade and year
       + (1|ITEM:outvar) ## variation in how each teacher skill may be related to different VAMs
       + (1|DISTRICT:TIMING:outvar) ## variation with respect to the academic calendar and the timing of the assessments used in VAM (capturing district-level variation around test prep or instructional initiatives)
       + (1|DISTRICT:GRADE:SCHOOLYEAR_SP:outvar) ## variation of emphases and curricula for districts during different seasons of the year, with respect to the kind of outcome (capturing district-level variation around test prep or instructional initiatives)
       + (0 + RepDelta|NCTETID) ## variation by teacher, which have a random slope to control the teacher variation with difference between the BLUP-estimated scores of this observation and BLUP-estimated teacher average for a given teacher skill, preserving more of the teacher effect in the residual where the observation is more similar to the teacher's average.  
\end{lstlisting}

\subsection{Residualized VAM model parameters and results}
Figure \ref{fig:align_app}, which follows the same format and notation as Figure \ref{fig:alignment} and the following tables report out the results from the alternative estimation methods. Note the similarity in distribution shapes when estimated using the more sophisticated noise-controlling technique to simple alignment of just $tau$ from the main body.

\begin{table}[h]\label{tab:vamfit}
\centering
\caption{Measures of fit for Residualization}
\begin{tabular}{l|r}
\hline
Parameter & Fit\\
\hline

R2 (conditional) & 0.26\\
\hline
R2 (marginal) & 0.14\\
\hline
Sigma & 0.12\\
\hline
\end{tabular}
\end{table}

\begin{table}
\centering
\caption{Mean VAM-alignment estimates across models after noise control, including three additional items from replication study.}
\begin{tabular}{llr}
\toprule
Item & Estimate (CI) & p.value\\
\midrule
All & -0.156 (-0.205, -0.107) & 0\\
CLBM & -0.356 (-0.426, -0.287) & 0\\
CLINSTD & -0.139 (-0.21, -0.069) & 0\\
LANGIMP & -0.181 (-0.23, -0.131) & 0\\
REMED & -0.233 (-0.282, -0.183) & 0\\
\midrule
CLPC & -0.141 (-0.212, -0.071) & 0\\
EXPL & -0.212 (-0.261, -0.162) & 0\\
SMQR & -0.124 (-0.174, -0.074) & 0\\
\bottomrule
\end{tabular}
\end{table}

\begin{figure*}
    \centering
    \includegraphics[width=1\linewidth]{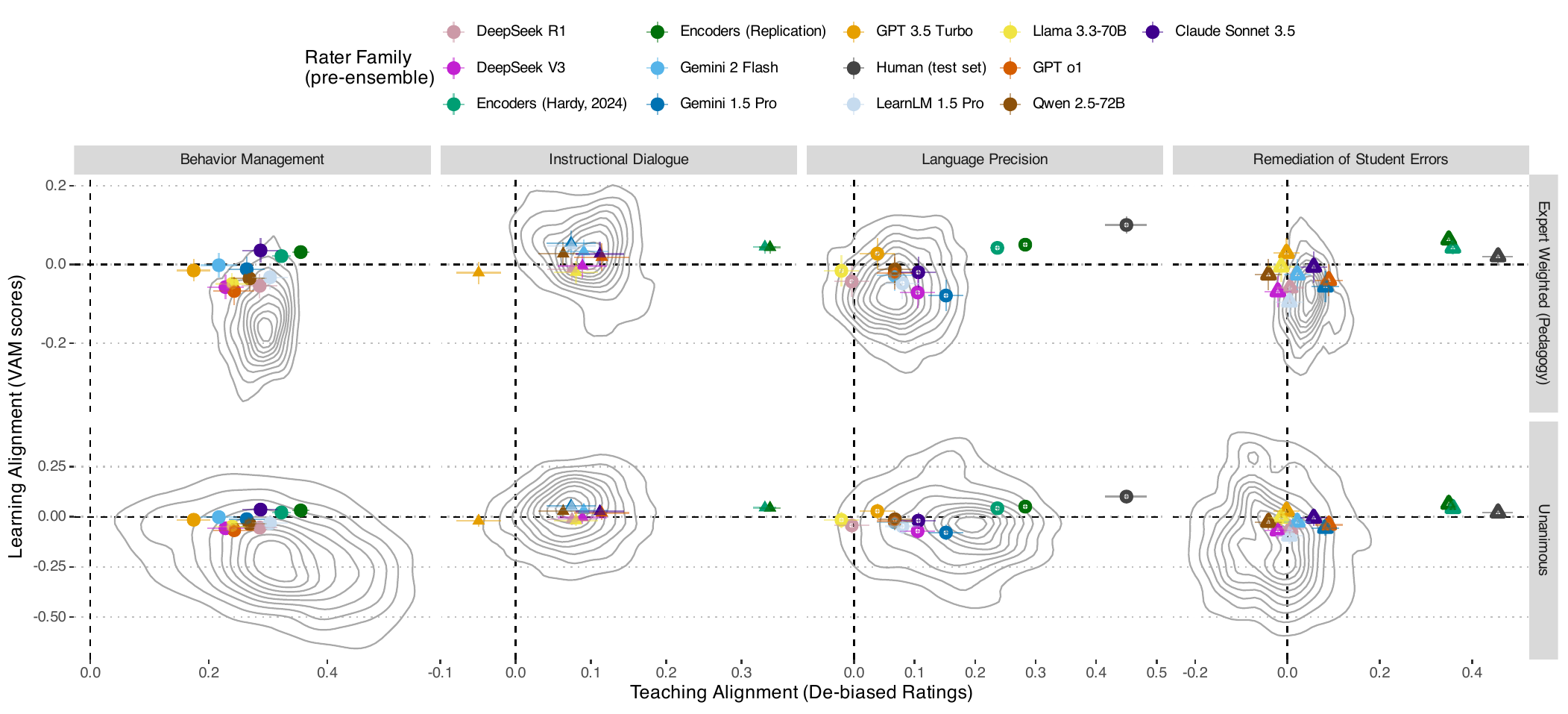}
    \caption{see the caption of Figure \ref{fig:alignment} for full description of details. The difference is that the axes have undergone the noise controlling transformations of this appendix prior to plotting.}
    \label{fig:align_app}
\end{figure*}

\begin{table*}
\centering
\begin{tabular}{lrlrr}
\toprule
Parameter & Coefficient & CI & t & p\\
\midrule
DISTRICT [11] & 0.41 & {}[0.15, 0.67] & 3.09 & 0.00\\
DISTRICT [12] & 0.20 & {}[-0.06, 0.46] & 1.54 & 0.12\\
DISTRICT [13] & 0.60 & {}[0.34, 0.86] & 4.52 & 0.00\\
DISTRICT [14] & 0.14 & {}[-0.12, 0.4] & 1.07 & 0.29\\
GRADE & -0.10 & {}[-0.16, -0.04] & -3.47 & 0.00\\
\addlinespace
V CS ALT IRT M TM1 & 0.14 & {}[0.14, 0.14] & 319.98 & 0.00\\
V CS STATE STD M TM1 & -0.02 & {}[-0.02, -0.02] & -34.29 & 0.00\\
V CS STATE STD E TM1 & -0.11 & {}[-0.11, -0.11] & -194.07 & 0.00\\
V CCLASS SIZE & 0.00 & {}[0, 0] & 9.34 & 0.00\\
MAXCHAP & 0.00 & {}[0, 0] & 66.59 & 0.00\\
\addlinespace
V CS SPED & 0.01 & {}[0.01, 0.01] & 11.58 & 0.00\\
V CS LEP & 0.03 & {}[0.03, 0.03] & 39.37 & 0.00\\
V CS FRPL & 0.03 & {}[0.03, 0.03] & 27.82 & 0.00\\
ACCURACY yr est & 0.01 & {}[0, 0.01] & 21.93 & 0.00\\
KOSM yr est & 0.00 & {}[0, 0] & -12.57 & 0.00\\
\addlinespace
MATH KNOWLEDGE yr est & 0.01 & {}[0.01, 0.01] & 88.31 & 0.00\\
samprep & 0.00 & {}[-0.01, 0] & -0.80 & 0.43\\
TIMING [WINTER] & 0.02 & {}[0.01, 0.04] & 2.51 & 0.01\\
TIMING [SPRING] & 0.03 & {}[0.01, 0.05] & 3.30 & 0.00\\
TIMING [FALL] & 0.04 & {}[0.02, 0.05] & 3.91 & 0.00\\
\addlinespace
V SS FRPL & 0.03 & {}[0.03, 0.04] & 25.00 & 0.00\\
V SS SPED & -0.16 & {}[-0.16, -0.15] & -90.17 & 0.00\\
V SS LEP & -0.02 & {}[-0.02, -0.02] & -15.78 & 0.00\\
V SS STATE STD M TM1 & -0.04 & {}[-0.04, -0.04] & -49.53 & 0.00\\
V SS STATE STD E TM1 & 0.05 & {}[0.05, 0.05] & 69.43 & 0.00\\
\addlinespace
V SSCHOOL SIZE & 0.00 & {}[0, 0] & -114.14 & 0.00\\
DISTRICT [12] × GRADE & 0.05 & {}[-0.03, 0.13] & 1.13 & 0.26\\
DISTRICT [13] × GRADE & -0.03 & {}[-0.11, 0.05] & -0.70 & 0.48\\
DISTRICT [14] × GRADE & 0.06 & {}[-0.02, 0.14] & 1.36 & 0.17\\
sd(NCTETID) & 0.04 &  & NA & NA\\
\addlinespace
sd(DISTRICT:GRADE:SCHOOLYEAR\_SP:outvar) & 0.02 &  & NA & NA\\
sd(DISTRICT:TIMING:outvar) & 0.03 &  & NA & NA\\
sd(DISTRICT:GRADE:SCHOOLYEAR\_SP) & 0.03 &  & NA & NA\\
sd(Residual) & 0.12 &  & NA & NA\\
\bottomrule
\end{tabular}
\end{table*}

\section{AI Use}
AI assistants (Gemini 2.5) was used during final revision to polish writing. We have mixed feelings about its ability to do the task well. We tried Claude for camera-ready polishing, which we mostly had to unpolish.

\end{document}